\documentclass[11pt]{article}

\usepackage{fullpage}
\RequirePackage{fancyhdr}
\RequirePackage{xcolor} 
\RequirePackage{algorithm}
\RequirePackage{algorithmic}
\RequirePackage{natbib}
\RequirePackage{eso-pic} 
\RequirePackage{forloop}
\RequirePackage{url}

\setcitestyle{authoryear,round,citesep={;},aysep={,},yysep={;}}

\usepackage{amsmath,amscd,amssymb,amsthm}
\usepackage{verbatim}
\usepackage{thmtools}
\usepackage{thm-restate}

\declaretheorem[name=Theorem]{Theorem}
\declaretheorem[name=Proposition, numberlike=Theorem]{Proposition}
\declaretheorem[name=Lemma, numberlike=Theorem]{Lemma}
\declaretheorem[sibling=Theorem]{Definition}
\declaretheorem[sibling=Theorem]{Corollary}

\usepackage{enumerate}
\usepackage{mathtools}
\usepackage{epsfig}
\usepackage{multirow}
\usepackage{graphicx}
\usepackage{epic}
\usepackage{eepic}
\usepackage{ifthen}
\usepackage{amsfonts}
\usepackage{wasysym}
\usepackage{color}
\usepackage{setspace}
\usepackage{booktabs}
\usepackage{diagbox}
\usepackage{thm-restate}

\usepackage{etoolbox}

\newtoggle{icml}

\usepackage{tikz}
\usetikzlibrary{shapes}

\usepackage{url}
\usepackage{hyperref}

\newcommand{\be}[1]{\begin{equation}\label{#1}}

\def\bnabla{{\boldsymbol \nabla}}
\def\bDelta{{\boldsymbol \Delta}}
\def\bgamma{{S}}
\def\bz{{\mathbf z}}
\def\bx{{\mathbf x}}

\def\be{{\mathbf e}}
\def\bh{{\mathbf h}}

\def\bw{{\mathbf w}}

\def\bd{{\mathbf d}}
\def\bv{{\mathbf v}}

\newcommand{\Pair}[2]{\boldsymbol{[\kern-.18em[}#1,#2\boldsymbol{]\kern-.18em]}}

\renewcommand{\O}{\textsc{Grad}}
\renewcommand{\H}{\mathcal{H}}

\newcommand{\Z}{\mathcal{Z}}
\newcommand{\Arand}{\textsc{A}_\text{rand}}
\newcommand{\prog}{\text{prog}}
\newcommand{\gap}{\gamma}

\newcommand{\E}{\mathop{\mathbb{E}}}
\newcommand{\R}{\mathbb{R}}
\newcommand{\A}{\mathcal{A}}
\newcommand{\N}{\mathbb{N}}

\newcommand{\argmin}{\mathop{\text{argmin}}}

\newcommand{\by}{\mathbf{y}}
\newcommand{\bg}{\mathbf{g}}
\newcommand{\bu}{\mathbf{u}}

\newtheorem{remark}[Theorem]{Remark}

\togglefalse{icml}
\title{Optimal Stochastic Non-smooth Non-convex Optimization through Online-to-Non-convex Conversion}

\author{
   Ashok Cutkosky \\
   Boston University \\
   Boston, MA \\
   \texttt{ashok@cutkosky.com}
   \and   
   Harsh Mehta \\
   Google Research \\
   Mountain View, CA \\
   \texttt{harshm@google.com}
   \and
   Francesco Orabona \\
   Boston University \\
   Boston, MA \\
   \texttt{francesco@orabona.com}
}

\date{}
\begin{document}

\maketitle

\begin{abstract}
We present new algorithms for optimizing non-smooth, non-convex stochastic objectives based on a novel analysis technique. This improves the current best-known complexity for finding a $(\delta,\epsilon)$-stationary point from $O(\epsilon^{-4}\delta^{-1})$ stochastic gradient queries to $O(\epsilon^{-3}\delta^{-1})$, which we also show to be optimal. Our primary technique is a reduction from non-smooth non-convex optimization to \emph{online learning}, after which our results follow from standard regret bounds in online learning. For \emph{deterministic and second-order smooth} objectives, applying more advanced optimistic online learning techniques enables a new complexity of $O(\epsilon^{-1.5}\delta^{-0.5})$. Our techniques also recover all optimal or best-known results for finding $\epsilon$ stationary points of smooth or second-order smooth objectives in both stochastic and deterministic settings.
\end{abstract}

\section{Introduction}
\label{sec:intro}

Algorithms for non-convex optimization are some of the most important tools in modern machine learning, as training neural networks requires optimizing a non-convex objective.
Given the abundance of data in many domains, the time to train a neural network is the current bottleneck to having bigger and more powerful machine learning models. Motivated by this need, the past few years have seen an explosion of research focused on understanding non-convex optimization~\citep{ghadimi2013stochastic, carmon2017convex,arjevani2019lower, arjevani2020second, carmon2019lower, fang2018spider}. Despite significant progress, key issues remain unaddressed.

In this paper, we work to minimize a potentially non-convex objective $F:\R^d\to\R$ which we only accesss in some stochastic or ``noisy'' manner.
As motivation, consider $F(\bx) \triangleq \E_{\bz}[f(\bx,\bz)]$, where $\bx$ can represent the model weights, $\bz$ a minibatch of i.i.d. examples, and $f$ the loss of a model with parameters $\bx$ on the minibatch $\bz$. In keeping with standard empirical practice, we will restrict ourselves to first order algorithms (gradient-based optimization).


The vast majority of prior analyses of non-convex optimization algorithms impose various smoothness conditions on the objective~\citep{ghadimi2013stochastic, carmon2017convex, allen2018natasha, tripuraneni2018stochastic, fang2018spider,zhou2018stochastic, fang2019sharp, cutkosky2019momentum,  li2019convergence, cutkosky2020momentum, zhang2020adaptive, karimireddy2020mime, levy2021storm+, faw2022power, liu2022meta}. One motivation for smoothness assumptions is that they allow for a convenient surrogate for global minimization: rather than finding a global minimum of a neural network's loss surface (which may be intractable), we can hope to find an $\epsilon$-stationary point, i.e., a point $\bx$ such that $\|\nabla F(\bx)\|\le \epsilon$. By now, the fundamental limits on first order smooth non-convex optimization are well understood: Stochastic Gradient Descent (SGD) will find an $\epsilon$-stationary point in $O(\epsilon^{-4})$ iterations, which is the optimal rate~\citep{arjevani2019lower}. Moreover, if $F$ happens to be second-order smooth, SGD requires only $O(\epsilon^{-3.5})$ iterations, which is also optimal~\citep{fang2019sharp, arjevani2020second}. These optimality results motivate the popularity of SGD and its variants in practice~\citep{kingma2014adam, loshchilov2016sgdr, loshchilov2018decoupled, goyal2017accurate, you2019large}.

Unfortunately, many standard neural network architectures are non-smooth (e.g., architectures incorporating ReLUs or max-pools cannot be smooth). As a result, these analyses can only provide intuition about what might occur when an algorithm is deployed in practice: the theorems themselves do not apply (see \citet{patel2022gradient} for examples of failure of SGD in non-smooth settings, or \citet{li2021second} for futher discussion of assumptions). Despite the obvious need for non-smooth analyses, recent results suggest that even approaching a neighborhood of a stationary point may be impossible for non-smooth objectives~\citep{kornowski2022oracle}. Nevertheless, optimization clearly \emph{is} possible in practice, which suggests that we may need to rethink our assumptions and goals in order to understand non-smooth optimization.

Fortunately, \citet{zhang2020complexity} recently considered an alternative definition of stationarity that \emph{is} tractable even for non-smooth objectives and which has attracted much interest  \citep{davis2021gradient, tian2022finite, kornowski2022complexity, tian2022no, jordan2022complexity}. Roughly speaking, instead of $\|\nabla F(\bx)\|\le \epsilon$, we ask that there is a random variable $\by$ supported in a ball of radius $\delta$ about $\bx$ such that $\|\E[\nabla F(\by)]\|\le \epsilon$. We call such an $\bx$ an $(\delta,\epsilon)$-stationary point, so that the previous definition ($\|\nabla F(\bx)\|\le \epsilon$) is a $(0,\epsilon)$-stationary point. The current best-known complexity for identifying an $(\delta,\epsilon)$ stationary point is $O(\epsilon^{-4}\delta^{-1})$ stochastic gradient evaluations.

In this paper, we significantly improve this result: we can identify an $(\delta,\epsilon)$-stationary point with only $O(\epsilon^{-3}\delta^{-1})$ stochastic gradient evaluations. Moreover, we also show that this rate is optimal. Our primary technique is a novel \emph{online-to-non-convex conversion}: a connection between non-convex stochastic optimization and \emph{online learning}, which is a classical field of learning theory that already has a deep literature~\citep{cesabianchi06prediction,hazan2019introduction,orabona2019modern}. In particular, we show that an online learning algorithm that provides a shifting regret bound can be used to decide the update step, when fed with linear losses constructed using the stochastic gradients of the function $F$. By establishing this connection, we open new avenues for algorithm design in non-convex optimization and also motivate new research directions in online learning.

In sum, we make the following contributions:
\begin{itemize}
\item A reduction from non-convex non-smooth stochastic optimization to online learning: better online learning algorithms result in faster non-convex optimization. Applying this reduction to standard online learning algorithms allows us to identify an $(\delta,\epsilon)$ stationary point in $O(\epsilon^{-3}\delta^{-1})$ stochastic gradient evaluations. The previous best-known rate in this setting was $O(\epsilon^{-4}\delta^{-1})$.
\item We show that the $O(\epsilon^{-3}\delta^{-1})$ rate is optimal for all $\delta, \epsilon$ such that $\epsilon\le O(\delta)$.
\end{itemize}
Additionally, we prove important corollaries for smooth $F$:
\begin{itemize}
\item The $O(\epsilon^{-3}\delta^{-1})$ complexity implies the optimal $O(\epsilon^{-4})$ and $O(\epsilon^{-3.5})$ respective complexities for finding $(0,\epsilon)$-stationary points of smooth or second-order smooth objectives.

\item For deterministic and second-order smooth objectives, we obtain a rate of $O(\epsilon^{-3/2}\delta^{-1/2})$, which implies the best-known $O(\epsilon^{-7/4})$ complexity for finding $(0,\epsilon)$-stationary points.
\end{itemize}

\section{Definitions and Setup}
\label{sec:setup}

Here, we formally introduce our setting and notation. We are interested in optimizing real-valued functions $F:\H\to \R$ where $\H$ is a real Hilbert space (e.g., usually $\H=\R^d$). We assume $F^\star\triangleq\inf_{\bx} F(\bx)>-\infty$. We assume that $F$ is differentiable, but we do \emph{not} assume that $F$ is smooth. All norms $\|\cdot\|$ are the Hilbert space norm (i.e., the 2-norm) unless otherwise specified. As mentioned in the introduction, the motivating example to keep in mind in our development is the case  $F(\bx) = \E_{\bz}[f(\bx,\bz)]$.

Our algorithms access information about $F$ through a \emph{stochastic gradient oracle} $\O:\H\times \Z\to \R$. Given a point $\bx$ in $\H$, the oracle will sample an i.i.d. random variable $\bz\in \Z$ and return $\O(\bx,\bz)\in \H$ such that $\E[\O(\bx,\bz)]=\nabla F(\bx)$ and $\text{Var}(\O(\bx,\bz))\le \sigma^2$. 

In the following, we only consider functions satisfying the following mild regularity condition.
\begin{Definition}
We define a differentiable function $F:\H \to \R$ to be \textbf{well-behaved} if for all  $\bx,\by\in \H$, it holds that
\[
F(\by)-F(\bx) 
 =  \int_0^1 \! \langle \nabla F(\bx+t(\by-\bx)), \by-\bx\rangle\, \mathrm{d}t~.
\]
\end{Definition}
If $F$ happens to be differentiable and locally Lipschitz, then this assumption is simply the Fundamental Theorem of Calculus. Under this assumption, our results can be applied to improve the past results on non-smooth stochastic optimization.
In fact, Proposition~\ref{prop:wellbehaved} (proof in Appendix~\ref{sec:proof_wellbehaved}) below shows that for the wide class of functions that are locally Lipschitz (but possibly non-differentiable), applying an arbitrarily small perturbation to the function is sufficient to ensure both differentiability and well-behavedness. This result works via standard perturbation arguments similar to those used previously by \citet{davis2021gradient} (see also \citet{bertsekas1973stochastic, duchi2012randomized, flaxman2005online} for similar techniques in the convex setting). In practice we suspect that such perturbation  arguments are unnecessary: intuitively an algorithm is unlikely to query a point of  non-differentiability (see also \citet{bianchi2022convergence} for some formal evidence for this idea).


\begin{Proposition}\label{prop:wellbehaved}
Let $F:\R^d \to \R$ be locally Lipschitz with stochastic oracle $\O$ such that $\E_{\bz}[\O(\bx,\bz)]=\nabla F(\bx)$ whenever $F$ is differentiable. We have two cases:
\begin{itemize}
\item If $F$ is differentiable everywhere, then $F$ is well-behaved.
\item If $F$ is not differentiable everywhere, let $p>0$ be an arbitrary number and let $\bu$ be a random vector in $\R^d$ uniformly distributed on the unit ball. Define $\hat F(\bx) \triangleq \E_{\bu}[F(\bx+p \bu)]$. Then, $\hat F$ is differentiable and well-behaved, and the oracle $\widehat \O(\bx, (\bz ,\bu)) = \O(\bx+p\bu, \bz)$ is a stochastic gradient oracle for $\hat F$. Moreover, $F$ is differentiable at $\bx+p\bu$ with probability 1 and if $F$ is $G$-Lipschitz, then $|\hat F(\bx) - F(\bx)|\le p G $ for all $\bx$. 
\end{itemize}

\end{Proposition}

\begin{remark}
We explicitly note that our results cover the case in which $F$ is \emph{directionally} differentiable and we have access to a stochastic \emph{directional} gradient oracle, as considered by \citet{zhang2020complexity}. This is a less standard oracle $\O(\bx,\bv,\bz)$ that outputs $\bg$ such that $\E[\langle \bg, \bv\rangle]$ is the  \emph{directional derivative} of $F$ in the direction $\bv$. This setting is subtly different (although a directional derivative oracle is a gradient oracle at all points for which $F$ is continuously differentiable). In order to keep technical complications to a minimum, in the main text we consider the simpler stochastic gradient oracle discussed above. In Appendix~\ref{sec:directional}, we show that our results and techniques also apply using directional gradient oracles with only superficial modification.
\end{remark}


\subsection{$(\delta,\epsilon)$-Stationary Points}

Now, let us define our notion of $(\delta,\epsilon)$-stationary point. This definition is essentially the same as used in \citet{zhang2020complexity, davis2021gradient, tian2022finite}. It is in fact mildly more stringent since we restrict to distributions of finite support and require an ``unbiasedness'' condition in order to make eventual connections to second-order smooth objectives easier.
\begin{Definition}
\label{def:crit}
A point $\bx$ is an $(\delta,\epsilon)$-stationary point of an almost-everywhere differentiable function $F$ if there is a finite subset $\bgamma$ of the ball of radius $\delta$ centered at $\bx$ such that for $\by$ selected uniformly at random from $\bgamma$, $\E[\by]=\bx$ and $\|\E[\nabla F(\bz)]\|\le \epsilon$.
\end{Definition}

As a counterpart to this definition, we also define:
\begin{Definition}\label{def:critmeasure}
Given a point $\bx$, a number $\delta>0$ and a almost-everywhere differentiable function $F$, define 
\[
\|\nabla F(\bx)\|_\delta \triangleq \inf_{\bgamma\subset B(\bx,\delta), \frac{1}{|\bgamma|}\sum_{\by\in \bgamma} \by = \bx} \left\|\frac{1}{|\bgamma|}\sum_{\by\in \bgamma} \nabla F(\by)\right\|~.
\]
\end{Definition}

Let's also state an immediate corollary of Proposition~\ref{prop:wellbehaved} that converts a guarantee on a randomized smoothed function to one on the original function. This result is also immediate from Theorem 3.1 of \citet{lin2022gradient}.
\begin{Corollary}\label{cor:almostdiff}
Let $F:\R^d \to \R$ be $G$-Lipschitz. For $\epsilon>0$, let $p\le \delta$ and let $\bu$ be a random vector in $\R^d$ uniformly distributed on the unit ball. Define $\hat F(\bx) \triangleq \E_{\bu}[F(\bx+p \bu)]$. If a point $\bx$ satisfies $\|\nabla \hat{F}(\bx)\|_\delta\leq \epsilon$, then $\|\nabla F(\bx)\|_{2\delta}\leq \epsilon$.
\end{Corollary}

Our ultimate goal is to use $N$ stochastic gradient evaluations of $F$ to identify a point $\bx$ with as small a value of $\E[\|\nabla F(\bx)\|_\delta]$ as possible. For the rest of this paper we  will consider exclusively the case of well-behaved and differentiable objectives $F$. We focus our development on this conceptually simpler case in order to simplify the proofs as much as possible, however due to Proposition~\ref{prop:wellbehaved} and Corollary~\ref{cor:almostdiff}, our results will immediately extend from differentiable $F$ to those $F$ that are locally Lipschitz and for which $\O(\bx,\bz)$ returns a unbiased estimate of $\nabla F(\bx)$ whenever $F$ is differentiable at $\bx$.

\subsection{Online Learning}
\label{sec:online}

Here, we very briefly introduce the setting of online linear learning with shifting competitors, that will be the core of our online-to-non-convex conversion. We refer the interested reader to \citet{cesabianchi06prediction,hazan2019introduction, orabona2019modern} for a comprehensive introduction to online learning.
In the online learning setting, the learning process goes on in rounds. In each round the algorithm outputs a point $\bDelta_t$ in a feasible set $V$, and then receives a linear loss function $\ell_t(\cdot)= \langle \bg_t,\cdot\rangle$ and it pays $\ell_t(\bDelta_t)$. The goal of the algorithm is to minimize the \emph{static regret} over $T$ rounds, defined as the difference between its cumulative loss and the one of an arbitrary comparison vector $\bu \in V$:
\begin{equation*}
\label{eqn:static_regret}
    R_T(\bu) 
    \triangleq \sum_{t=1}^T \langle \bg_t, \bDelta_t - \bu\rangle~.
\end{equation*}
With \emph{no stochastic assumption}, it is possible to design online algorithms that guarantee that the regret is upper bounded by $O(\sqrt{T})$.
In this work, we frequently make use of a more challenging objective: minimizing the \emph{$K$-shifting regret}. This is the regret with respect to an arbitrary \emph{sequence} of 
$K$ vectors $\bu^1,\dots,\bu^K \in V$ that changes every $T$ iterations:
\begin{equation}
\label{eqn:regret}
    R_T(\bu^1,\dots, \bu^K) 
    \triangleq \sum_{k=1}^K \sum_{n=(k-1)T+1}^{kT} \langle \bg_n, \bDelta_n - \bu^k\rangle~.
\end{equation}
It should be intuitive that resetting the online algorithm every $T$ iterations can achieve a  shifting regret of $O(K\sqrt{T})$.

\subsection{Related Work}

In addition to the papers discussed in the introduction, here we discuss further related work.

In this paper we build on top of the definition of $(\delta, \epsilon)$-stationary points proposed by \cite{zhang2020complexity}. There, they prove a complexity rate of $O(\epsilon^{-4} \delta^{-1})$ for stochastic Lipschitz functions, which we improve to $O(\epsilon^{-3} \delta^{-1})$ and prove the optimality of this result.

In a concurrent work, \cite{chen2023faster} consider the setting of \emph{zeroth-order} stochastic optimization (i.e. evaluation of function values only rather than gradients) and achieve a similar asymptotic rate of $O(d^{3/2} \epsilon^{-3}\delta^{-1})$ by applying variance-reduction to a smoothed version of the objective $F$. This result is an intriguing contrast to many zeroth-order algorithms based on such smoothing in that the algorithm is not obtained by applying applying smoothing to a first-order algorithm. More recently, \cite{kornowski2023algorithm} improved the dimension dependence in the rate for zeroth-order optimization to $O(d\epsilon^{-3}\delta^{-1})$ by employing the algorithm described in this paper in concert with a refined analysis of the smoothing operation.

The idea to reduce machine learning to online learning was pioneered by \citet{cesa2004generalization} with the online-to-batch conversion. There is also previous work exploring the possibility of transforming non-convex problems into online learning ones. \citet{ghai2022non} provides some conditions under which online gradient descent on non-convex losses is equivalent to a convex online mirror descent. \citet{hazan2017efficient} defines a notion of regret which can be used to find approximate stationary points of smooth objectives. \citet{zhuang2019surrogate} transform the problem of tuning of learning rates in stochastic non-convex optimization into an online learning problem. Our proposed approach differs from all the ones above in applying to non-smooth objectives. Moreover, as discusses in the next section, we employ online learning algorithms with \emph{shifting} regret~\cite{HerbsterW98b} to generate the \emph{updates} (i.e. the differences between successive iterates), rather than the iterates themselves.

\section{Online-to-Non-Convex Conversion}

In this section, we explain the online-to-non-convex conversion.
The core idea transforms the minimization of a non-convex and non-smooth function onto the problem of minimizing the shifting regret over linear losses. In particular, consider an optimization algorithm that updates a previous iterate $\bx_{n-1}$ by moving in a direction $\bDelta_n$: $\bx_n=\bx_{n-1} + \bDelta_n$. For example, SGD sets $\bDelta_n=-\eta \bg_{n-1}=-\eta\cdot \O(\bx_{n-1},\bz_{n-1})$ for a learning rate $\eta$. Instead, we let an online learning algorithm $\A$  decide the update direction $\bDelta_n$, using linear losses $\ell_n(\bx)= \langle \bg_n, \bx\rangle$.

The motivation behind essentially all first order algorithms is that $F(\bx_{n-1} + \bDelta_n) -F(\bx_{n-1}) \approx \langle \bg_n, \bDelta_n\rangle $. This suggests that $\bDelta_n$ should be chosen to minimize the inner product $\langle \bg_n, \bDelta_n\rangle$. However, we are faced with two difficulties. The first difficulty is the the approximation error in the first-order expansion. The second is the fact that $\bDelta_n$ needs to be chosen before $\bg_n$ is revealed, so that $\bDelta_n$ needs in some sense to ``predict the future''. Typical analysis of algorithms such as SGD use the remainder form of Taylor's theorem to address both difficulties simultaneously for smooth objectives, but in our non-smooth case this is not a valid approach. Instead, we tackle these difficulties independently. We overcome the first difficulty using the same randomized scaling trick employed by \citet{zhang2020complexity}: define $\bg_n$ to be a gradient evaluated not at $\bx_{n-1}$ or $\bx_{n-1} +\bDelta_n$, but at a \emph{random} point along the line segment connecting the two. Then for a well-behaved function we will have $F(\bx_{n-1} + \bDelta_n) -F(\bx_{n-1}) = \E[\langle \bg_n, \bDelta_n\rangle]$. The second difficulty is where online learning shines: online learning algorithms are \emph{specifically designed} to predict \emph{completely arbitrary} sequences of vectors as accurately as possible.

The previous intuition is formalized in Algorithm~\ref{alg:template_no_epochs} and the following result, which we will elaborate on in Theorem~\ref{thm:nonsmooth} before yielding our main result in Corollary~\ref{thm:nonsmoothogd}.
\begin{Theorem}
\label{thm:otnc}
Suppose $F$ is well-behaved. Define $\bnabla_n = \int_0^1 \! \nabla F(\bx_{n-1} + s\bDelta_n)\, \mathrm{d}s$. 
Then, with the notation in Algorithm~\ref{alg:template_no_epochs} and for any sequence of vectors $\bu_1, \dots, \bu_N$, we have the \emph{equality}:
\begin{align*}
    F(\bx_M)&=F(\bx_0) + \sum_{n=1}^M \langle \bg_n, \bDelta_n - \bu_n\rangle
    \iftoggle{icml}{\\ &\quad +}{+}
     \sum_{n=1}^M \langle \bnabla_n - \bg_n,\bDelta_n\rangle+\sum_{n=1}^M \langle \bg_n, \bu_n\rangle~.
\end{align*}
Moreover, if we let $s_n$ be independent random variables uniformly distributed in $[0,1]$, then we have
\begin{align*}
    \E[F(\bx_M)]&= F(\bx_0) + \E\left[\sum_{n=1}^M \langle \bg_n, \bDelta_n - \bu_n\rangle\right] 
    \iftoggle{icml}{\\ &\quad +}{+}
     \E\left[\sum_{n=1}^M \langle \bg_n, \bu_n\rangle\right]~.
\end{align*}
\end{Theorem}
\begin{proof}
By the well-behaveness of $F$, we have
\begin{align*}
    F
    \iftoggle{icml}{%
    & (\bx_n) - F(\bx_{n-1}) \\
    }{%
    (\bx_n) - F(\bx_{n-1}) }
    &=\int_0^1 \! \langle \nabla F(\bx_{n-1} + s(\bx_n - \bx_{n-1})),\bx_n-\bx_{n-1}\rangle \, \mathrm{d}s\\
    &= \int_0^1 \!\langle \nabla F(\bx_{n-1} + s\bDelta_n), \bDelta_n\rangle\, \mathrm{d}s\\
    &= \langle \bnabla_n,\bDelta_n\rangle\\
    &= \langle \bg_n, \bDelta_n-\bu_n\rangle + \langle \bnabla_n -\bg_n, \bDelta_n\rangle + \langle \bg_n,\bu_n\rangle~.
\end{align*}
Now, sum over $n$ and telescope to obtain the stated bound.

For the second statement, simply observe that by definition we have $\E[ \bg_n]=\int_0^1 \! \nabla F(\bx_{n-1} + s\bDelta_n) \, \mathrm{d}s = \bnabla_n$.
\end{proof}

\begin{algorithm}[t]
   \caption{Online-to-Non-Convex Conversion}
   \label{alg:template_no_epochs}
  \begin{algorithmic}
      \STATE{\bfseries Input: } Initial point $\bx_0$, $K \in\N$, $T \in \N$, online learning algorithm $\A$.
      \STATE Set $M=K \cdot T$
      \FOR{$n=1\dots M$}
      \STATE Get $\bDelta_n$ from $\A$
      \STATE Set $\bx_n =\bx_{n-1}+\bDelta_n$
      \STATE Generate $s_n\in[0,1]$ // usually uniformly random, see Theorem statements for precise settings.

      \STATE Set $\bw_n = \bx_{n-1} + s_n\bDelta_n$
      \STATE Sample random $\bz_n$
      \STATE Generate gradient $\bg_n= \O(\bw_n, \bz_n)$
      \STATE Send $\bg_n$ to $\A$ as gradient
      \ENDFOR
      \STATE Set $\bw^k_t=\bw_{(k-1)T+t}$ for $k=1,\dots, K$ and $t=1,\dots, T$
      \STATE Set $\overline{\bw}^k = \frac{1}{T}\sum_{t=1}^T \bw^k_t$  for $k=1,\dots, K$
      \STATE{\bfseries Return} $\{\overline{\bw}^1, \dots, \overline{\bw}^K\}$
   \end{algorithmic}
\end{algorithm}

\subsection{Guarantees for Non-Smooth Non-Convex Functions}
The primary value of Theorem~\ref{thm:otnc} is that the term $\sum_{n=1}^M \langle \bg_n, \bDelta_n - \bu_n\rangle$ is exactly the regret of an online learning algorithm: lower regret clearly translates to a smaller bound on $F(\bx_M)$. Next, by carefully choosing $\bu_n$, we will be able to relate the term $\sum_{n=1}^M \langle \bg_n ,\bu_n\rangle$ to the gradient averages that appear in the definition of $(\delta, \epsilon)$-stationarity. 
Formalizing these ideas, we have the following:



\begin{Theorem}
\label{thm:nonsmooth}
Assume $F$ is well-behaved.
With the notation in  Algorithm~\ref{alg:template_no_epochs}, set $s_n$ to be a random variable sampled uniformly from $[0,1]$. Set $T,K \in \N$ and $M=KT$.
Define $\bu^k = -D\frac{\sum_{t=1}^T\nabla F\left( \bw^k_t\right)}{\left\|\sum_{t=1}^T \nabla F\left( \bw^k_t\right)\right\|}$ for some $D>0$ for $k=1,\dots,K$.
Finally, suppose $\text{Var}(\bg_n)\le \sigma^2$. Then:
\begin{align*}
\E
\iftoggle{icml}{ & 
\left[\frac{1}{K}\sum_{k=1}^K \left\|\frac{1}{T}\sum_{t=1}^T \nabla F(\bw^k_t)\right\|\right] \\ }{%
\left[\frac{1}{K}\sum_{k=1}^K \left\|\frac{1}{T}\sum_{t=1}^T \nabla F(\bw^k_t)\right\|\right]}
&\le \frac{F(\bx_0)-F^\star}{D M} + \frac{\E[R_T(\bu^1,\dots,\bu^K)]}{D M}+\frac{\sigma}{\sqrt{T}}~.
\end{align*}
\end{Theorem}

\begin{proof}
From Theorem~\ref{thm:otnc}, we have
\begin{align*}
    \E[F(\bx_M)]
    &= F(\bx_0) + \E[R_T(\bu^1,\dots,\bu^K)]
    \iftoggle{icml}{ \\ &\quad +}{+}
     \E\left[\sum_{n=1}^M\langle \bg_n, \bu_n\rangle\right]~.
\end{align*}
Now, since $\bu^k=-D \frac{\sum_{t=1}^T\nabla F(\bw^k_t)}{\left\|\sum_{t=1}^T \nabla F(\bw^k_t)\right\|}$, $\E[\bg_n] = \E[\nabla F(\bw_n)]$, and $\text{Var}(\bg_n)\le \sigma^2$, we have
\begin{align*}
    \iftoggle{icml}{ &\E\left[\sum_{n=1}^M\langle \bg_n, \bu_n\rangle\right]}{%
    \E\left[\sum_{n=1}^M\langle \bg_n, \bu_n\rangle\right]&}
    \le \E\left[\sum_{k=1}^K\left\langle \sum_{t=1}^T \nabla F(\bw_t^k), \bu^k\right\rangle\right]
    \iftoggle{icml}{ \\ &\quad +}{+}
     \E\left[D \sum_{k=1}^K\left\|\sum_{t=1}^T(\nabla F(\bw_t^k) -\bg_{T(k-1)+t})\right\| \right]\\
    &\le \E\left[\sum_{k=1}^K\left\langle \sum_{t=1}^T \nabla F(\bw_t^k), \bu^k\right\rangle\right] +D \sigma K\sqrt{T}\\
    &= \E\left[-\sum_{k=1}^K DT\left\|\frac{1}{T}\sum_{t=1}^T \nabla F\left( \bw^k_t\right)\right\|\right] +D \sigma K\sqrt{T}~.
\end{align*}

Putting this all together, we have
\begin{align*}
F^\star 
&\leq \E[F(\bx_M)]
\le  F(\bx_0) + \E[R_T(\bu^1,\dots, \bu^K)] 
\iftoggle{icml}{ \\ &\quad + }{+}
\sigma DK\sqrt{T}-DT \sum_{k=1}^K\E\left[\left\|\frac{1}{T}\sum_{t=1}^T \nabla F\left( \bw^k_t\right)\right\|\right]~.
\end{align*}
Dividing by $K D T = D M$ and reordering, we have the stated bound.
\end{proof}

We now instantiate Theorem~\ref{thm:nonsmooth} with the simplest online learning algorithm: online gradient descent (OGD)~\citep{zinkevich2003online}. OGD takes input a radius $D$ and a step size $\eta$ and makes the update $\bDelta_{n+1} = \Pi_{\|\bDelta\|\le D} [\bDelta_n - \eta \bg_n]$ with $\bDelta_1=0$. The standard analysis shows that if $\E[\|\bg_n\|^2]\le G^2$ for all $n$, then with $\eta=\frac{D}{G\sqrt{T}}$, OGD will ensure\footnote{For completeness a proof of this statement is in Appendix~\ref{sec:online_algos}.} static regret $\E[R_T(\bu) ]\le DG\sqrt{T}$ for \emph{any}  $\bu$ satisfying $\|\bu\|\le D$. Thus, by resetting the algorithm every $T$ iterations, we achieve $\E[R_T(\bu^1,\dots\bu^K)]\le K D G\sqrt{T}$. This powerful guarantee for all sequences is characteristic of online learning. We are now free to optimize the remaining parameters $K$ and $D$ to achieve our main result, presented in Corollary~\ref{thm:nonsmoothogd}.
\begin{Corollary}\label{thm:nonsmoothogd}
Suppose we have a budget of $N$ gradient evaluations.
Under the assumptions of Theorem~\ref{thm:nonsmooth}, suppose in addition $\E[\|\bg_n\|^2]\le G^2$ and that $\A$ guarantees $\|\bDelta_n\|\le D$ for some user-specified $D$ for all $n$ and ensures the worst-case $K$-shifting regret bound $\E[R_T(\bu^1, \dots, \bu^K)]\le DGK\sqrt{T}$ for all $\|\bu^k\|\le D$ (e.g., as achieved by the OGD algorithm that is reset every $T$ iterations). Let $\delta>0$ be an arbitrary number. Set $D=\delta/T$, $T=\min(\lceil(\frac{GN\delta}{F(\bx_0)-F^\star})^{2/3}\rceil,\frac{N}{2})$, and $K=\lfloor\frac{N}{T}\rfloor$.
Then,
for all $k$ and $t$, we have $\|\overline{\bw}^k-\bw^k_t\|\le \delta$.

Moreover, we have the inequality
\begin{align*}
\E&\left[\frac{1}{K}\sum_{k=1}^K \left\|\frac{1}{T}\sum_{t=1}^T \nabla F(\bw^k_t)\right\|\right] \leq \frac{2(F(\bx_0)-F^\star)}{\delta N}
\iftoggle{icml}{ \\ & +}{+}
 \max\left(\frac{5G^{2/3}(F(\bx_0)-F^\star)^{1/3}}{(N\delta)^{1/3}},\frac{6G}{\sqrt{N}}\right), 
\end{align*}
which implies
\begin{align*}
\iftoggle{icml}{\E&}{\E}\left[\frac{1}{K}\sum_{t=1}^K \|\nabla F(\overline{\bw}^k)\|_\delta \right]\leq \frac{2(F(\bx_0)-F^\star)}{\delta N}
\iftoggle{icml}{ \\ & +}{+}
\max\left(\frac{5G^{2/3}(F(\bx_0)-F^\star)^{1/3}}{(N\delta)^{1/3}},\frac{6G}{\sqrt{N}}\right)~. 
\end{align*}
\end{Corollary}
Before providing the proof, let us discuss the implications. Notice that if we select $\hat \bw$ at random from $\{\overline{\bw}^1,\dots,\overline{\bw}^K\}$, then we clearly have $\E[\|\nabla  F(\hat \bw)\|_{\delta}]=\E\left[\frac{1}{K}\sum_{t=1}^K \|\nabla F(\overline{\bw}^k)\|_\delta\right]$. Therefore, the Corollary asserts that for a function $F$ with $F(\bx_0)-\inf F(\bx)\le \gap$ with a stochastic gradient oracle whose second moment is bounded by $G^2$, a properly instantiated Algorithm~\ref{alg:template_no_epochs} finds a $(\delta,\epsilon)$ stationary point in $N=O(G \gap \epsilon^{-3}\delta^{-1})$ gradient evaluations. In Section~\ref{sec:mainlowerbounds}, we will provide a lower bound showing that this rate is optimal essentially whenever $\delta G^2 \ge \epsilon \gap$. Together, the Corollary and the lower bound provide a nearly complete characterization of the complexity of finding $(\delta, \epsilon)$-stationary points in the stochastic setting.

It is also interesting to note that the bound does not appear to improve if the gradients are deterministic. Specifically, in the assumptions for Corollary~\ref{thm:nonsmoothogd}, we could try to relax $\E[\|\bg_t\|^2]\le G$ to $\|\nabla F(\bw_t)\|\le G$ and $\text{Var}(\bg_t)\le \sigma^2$ for some $\sigma$. We might then hope to improve the bound as $\sigma\to 0$ by taking advantage of the $\sigma$-dependency in Theorem~\ref{thm:nonsmooth}.  However, it turns out that the $\sigma$-dependency in Corollary~\ref{thm:nonsmoothogd} is dominated by a dependency on $G$ coming from the regret bound of OGD. This highlights an interesting open question: is it actually possible to improve in the deterministic setting? It is conceivable that the answer is ``no'': in the non-smooth \emph{convex} optimization setting,  it is well-known that the optimal rates for stochastic and  deterministic optimization are the same  (see, e.g., \citet{bubeck2015convex} for proofs of both upper and lower bounds).

\begin{remark}
We conjecture that by employing martingale concentration, the above can be extended to identify a $(\delta, O(\frac{G^{2/3}(F(\bx_0)-F^\star)^{1/3}}{(N\delta)^{1/3}}))$-stationary point with high probability, although we do not establish such a result here. 
\end{remark}



It is also interesting to explicitly write  the update of the overall algorithm:
\begin{align*}
\bx_n&=\bx_{n-1}+\bDelta_n\\
\bg_n &= \O(\bx_n+ (s_n-1)\bDelta_n,\bz_n)\\
\bDelta_{n+1}&=\text{clip}_D(\bDelta_n+\eta\bg_n)
\end{align*}
where clip$_D(\bx)=\bx\min(\frac{D}{\|\bx\|},1)$. In words, the update is reminiscent of the SGD update with momentum and clipping. The primary different element is the fact that the stochastic gradient is taken on a slightly perturbed $\bx_n$.

\begin{proof}[Proof of Corollary~\ref{thm:nonsmoothogd}]
Since $\A$ guarantees $\|\bDelta_n\|\le D$, for all $n<n'\leq n+T-1$, we have
\begin{align*}
\|\bw_n - \bw_{n'}\|&=\|\bx_{n}-(1-s_n)\bDelta_n -\bx_{n'-1} + s_{n'}\bDelta_{n'}\|\\
&\le \left\|\sum_{i=n+1}^{n'-1} \bDelta_i\right\| +\|\bDelta_n\|+\|\bDelta_{n'}\|\\
&\le D((n'-1) -(n+1) +1) + 2D\le DT~.
\end{align*}
Therefore, we clearly have $\|\bw^k_t-\overline{\bw}^k\|\le DT=\delta$.

Note that from the choice of $K$ and $T$ we have $M=KT\geq N-T\geq N/2$.
So, for the second fact, notice that $\text{Var}(\bg_n)\le\E[\|\bg_t\|^2]\le G^2$ for all $n$. Thus, applying Theorem~\ref{thm:nonsmooth} in concert with the additional assumption $\E[R_T(\bu^1,\dots, \bu^K)]\le DGK\sqrt{T}$, we have:
\begin{align*}
\E
\iftoggle{icml}{&\left[\frac{1}{K}\sum_{k=1}^K \left\|\frac{1}{T}\sum_{t=1}^T \nabla F(\bw^k_t)\right\|\right]\\}{%
\left[\frac{1}{K}\sum_{k=1}^K \left\|\frac{1}{T}\sum_{t=1}^T \nabla F(\bw^k_t)\right\|\right]}
&\le 2\frac{F(\bx_0)-F^\star}{DN} + 2\frac{KDG\sqrt{T}}{DN}+\frac{G}{\sqrt{T}}\\
&\le \frac{2T(F(\bx_0)-F^\star)}{\delta N} +\frac{3G}{\sqrt{T}}\\
&\le \max\left(\frac{5G^{2/3}(F(\bx_0)-F^\star)^{1/3}}{(N\delta)^{1/3}}, \frac{6G}{\sqrt{N}}\right)
\iftoggle{icml}{ \\ &\quad +}{+}
 \frac{2(F(\bx_0)-F^\star)}{\delta N},
\end{align*}
where the last inequality is due to the choice of $T$.

Finally, observe that $\|\bw^k_t-\overline{\bw}^k\|\le \delta$ for all $t$ and $k$, and also that $\overline{\bw}^k =\frac{1}{T}\sum_{t=1}^T \bw^k_t$. Therefore $\bgamma=\{\bw^k_1,\dots,\bw^k_T\}$ satisfies the conditions in the infimum in Definition~\ref{def:critmeasure} so that $\|\nabla F(\overline{\bw}^k)\|_\delta \le \left\|\frac{1}{T}\sum_{t=1}^T \nabla F(\bw^k_t)\right\|$.
\end{proof}

\section{Bounds for the $L_1$ Norm}
It is a well-known trick in the online learning literature that running a separate instance of an online learning algorithm on each coordinate of $\bDelta$ yields regret bounds with respect to $L_1$ norms of the linear costs (e.g., as in AdaGrad  \citep{duchi10adagrad,mcmahan2010adaptive}). For example, we can run the online gradient descent algorithm with a separate learning rate for each coordinate: $\bDelta_{n+1,i} = \Pi_{[-D_\infty, D_\infty]}[\bDelta_{n,i}-\eta_i \bg_{n,i}]$. The regret of this procedure is simply the sum of the regrets of each of the individual algorithms. In particular, if $\E[\bg_{n,i}^2]\le G_i^2$, then setting $\eta_i = \frac{D_\infty}{G_i\sqrt{T}}$ yields the regret bound $\E[R_T(\bu)]\le D_\infty \sqrt{T}\sum_{i=1}^N G_i$ for any $\bu$ satisfying  $\|\bu\|_{\infty}\le D_\infty$. By employing such online algorithms with our online-to-non-convex conversion, we can obtain a guarantee on the $L_1$ norm of the gradients. 
%
\begin{Definition}
\label{def:crit_coordinate}
A point $\bx$ is a $(\delta,\epsilon)$-stationary point with respect to the $L_1$ norm of an almost-everywhere differentiable function $F$ if there exists a finite subset $\bgamma$ of the $L_\infty$ ball of radius $\delta$ centered at $\bx$ such that if $\by$ is selected uniformly at random from $\bgamma$, $\E[\by]=\bx$ and $\|\E[\nabla F(\by)]\|_1\le \epsilon$.
\end{Definition}

As a counterpart to this definition, we define:
\begin{Definition}\label{def:critmeasure_coordinate}
Given a point $\bx$, a number $\delta>0$ and an almost-everywhere differentiable function $F$, define 
\begin{align*}
\|\nabla F(\bx)\|_{1,\delta} \triangleq \inf_{\bgamma\subset B_\infty(\bx,\delta)|, \frac{1}{|\bgamma|}\sum_{\by\in \bgamma} \by = \bx} \left\|\frac{1}{|\bgamma|}\sum_{\by\in \bgamma} \nabla F(\by)\right\|_1~.
\end{align*}
\end{Definition}

We now can state a theorem similar to Corollary~\ref{thm:nonsmoothogd}. Given that the proof is also very similar, we defer it to Appendix~\ref{sec:proof_nonsmoothogd_coordinate}.
\begin{Theorem}
\label{thm:nonsmoothogd_coordinate}
Suppose we have a budget of $N$ gradient evaluations.
Assume $F:\R^d \to \R$ is well-behaved. With the notation in Algorithm~\ref{alg:template_no_epochs}, set $s_n$ to be a random variable sampled uniformly from $[0,1]$. Set $T,K \in \N$ and $M=KT$. Assume that $\E[g_{n,i}^2] \leq G^2_i$ for $i=1, \dots, d$ for all $n$. Assume that $\A$ guarantees $\|\bDelta_n\|_\infty\le D_\infty$ for some user-specified $D_\infty$ for all $n$ and ensures the $K$-shifting regret bound $\E[R_T(\bu^1, \dots, \bu^K)]\le D_\infty K \sqrt{T} \sum_{i=1}^d G_i$ for all $\|\bu^k\|_\infty\le D_\infty$. Let $\delta>0$ be an arbitrary number. Set $D_\infty=\delta/T$, $T=\min(\lceil(\frac{N\delta \sum_{i=1}^d G_i }{F(\bx_0)-F^\star})^{2/3}\rceil,\frac{N}{2})$, and $K=\lfloor\frac{N}{T}\rfloor$.
Then we have:

\begin{align*}
&\frac{1}{K}\sum_{t=1}^K \|\nabla F(\overline{\bw}^k)\|_{1,\delta}
\le \frac{2(F(\bx_0)-F^\star)}{\delta N}
\iftoggle{icml}{ \\& +}{+}
\max\left(\frac{5(\sum_{i=1}^d G_i)^{2/3}(F(\bx_0)-F^\star)^{1/3}}{(N\delta)^{1/3}}, \frac{6 \sum_{i=1}^d G_i}{\sqrt{N}}\right).
\end{align*}
\end{Theorem}
Let's compare this result with Corollary~\ref{thm:nonsmoothogd}.
For a fair comparison, we set $G_i$ and $G$ such that $\sum_{i=1}^d G_i^2 = G^2$. Then, we can lower bound $\|\cdot\|_\delta$ with $\frac{1}{\sqrt{d}}\|\cdot\|_{1,\delta}$. Hence, under the assumption $\E[\|\bg_n\|^2]= \sum_{i=1}^d \E[g_{n,i}^2] \leq \sum_{i=1}^d G_i^2 = G^2$, Corollary~\ref{thm:nonsmoothogd} implies $\frac{1}{K}\sum_{t=1}^K \|\nabla F(\overline{\bw}^k)\|_{1,\delta}\leq O(\frac{G^{2/3}\sqrt{d}}{(N\delta)^{1/3}})$.

Now, let us see what would happen if we instead employed the above Corollary~\ref{thm:nonsmoothogd_coordinate}. First, observe that $\sum_{i=1}^d G_i \leq \sqrt{d} \sqrt{\sum_{i=1}^d G_i^2} \leq \sqrt{d} G$.
Substituting this expression into Theorem~\ref{thm:nonsmoothogd_coordinate} now gives an upper bound on $\|\cdot\|_{1,\delta}$ that is  $O(\frac{d^{1/3} G^{2/3}}{(N \delta)^{1/3}})$, which is better than the one we could obtain from Corollary~\ref{thm:nonsmoothogd} under the same assumptions.

\section{From Non-smooth to Smooth Guarantees}
\label{sec:to_smooth}

Let us now see what our results imply for smooth objectives. The following two propositions show that for smooth $F$, a $(\delta, \epsilon)$-stationary point is automatically a $(0,\epsilon')$-stationary point for some appropriate $\epsilon'$. The proofs are in Appendix~\ref{sec:proof_to_smooth}.
\begin{restatable}{Proposition}{nonsmoothtosmooth}\label{prop:nonsmoothtosmooth}
Suppose that $F$ is $H$-smooth (that is, $\nabla F$ is $H$-Lipschitz) and $x$ also satisfies $\|\nabla F(\bx)\|_\delta \le \epsilon$. Then, $\|\nabla F(\bx)\|\le \epsilon + H\delta$.
\end{restatable}

\begin{restatable}{Proposition}{nonsmoothtosecondsmooth}\label{prop:nonsmoothtosecondsmooth}
Suppose that $F$ is $J$-second-order-smooth (that is, $\|\nabla^2 F(\bx)-\nabla^2F(\by)\|_{\text{op}}\le J\|\bx-\by\|$ for all $\bx$ and $\by$). Suppose also that $\bx$ satisfies $\|\nabla F(\bx)\|_\delta \le \epsilon$. Then, $\|\nabla F(\bx)\|\le \epsilon + \frac{J}{2}\delta^2$. 
\end{restatable}

Now, recall that Corollary~\ref{thm:nonsmoothogd} shows that we can find a $(\delta,\epsilon)$ stationary point in $O(\epsilon^{-3}\delta^{-1})$ iteration. Thus, Proposition~\ref{prop:nonsmoothtosmooth} implies that by setting $\delta = \epsilon/H$, we can find a $(0,\epsilon)$-stationary point of an $H$-smooth objective $F$ in $O(\epsilon^{-4})$ iterations, which matches the  (optimal) guarantee of standard SGD~\citep{ghadimi2013stochastic, arjevani2019lower}. Further, Proposition~\ref{prop:nonsmoothtosecondsmooth} shows that by setting $\delta = \sqrt{\epsilon/J}$, we can find a $(0,\epsilon)$-stationary point of a $J$-second order smooth objective in $O(\epsilon^{-3.5})$ iterations. This matches the performance of more refined SGD variants and is also known to be tight~\citep{fang2019sharp,cutkosky2020momentum,arjevani2020second}. In summary: \textbf{the online-to-non-convex conversion also recovers the optimal results for smooth stochastic losses.}


\section{Deterministic and Smooth Case}

We will now consider the case of a non-stochastic oracle (that is, $\O(\bx,\bz) = \nabla F(\bx)$ for all $\bz$, $\bx$) and $F$ is $H$-smooth (i.e. $\nabla F$ is $H$-Lipschitz). We will show that \emph{optimistic} online algorithms~\citep{rakhlin2013online, hazan2010extracting} achieve rates matching the optimal deterministic results. In particular, we consider online algorithms that ensure static regret: 
\begin{align}
    R_T(\bu)
    \le O\left(D \sqrt{\sum_{t=1}^T \|\bh_{t}-\bg_{t}\|^2}\right), \label{eqn:omd_bound}
\end{align}
for some ``hint'' vectors $\bh_t$. In Appendix~\ref{sec:online_algos}, we provide an explicit construction of such an algorithm for completeness. The standard setting for the hints is $\bh_t=\bg_{t-1}$. As explained in Section~\ref{sec:online}, to obtain a $K$-shifting regret it will be enough to reset the algorithm every $T$ iterations.


\begin{Theorem}
\label{thm:detsmooth}
Suppose we have a budget of $N$ gradient evaluations.
and that we have an online algorithm $\A_{static}$ that guarantees $\|\bDelta_n\|\le D$ for all $n$ and ensures the optimistic regret bound $R_T(\bu)\le CD\sqrt{\sum_{t=1}^T \|\bg_t-\bg_{t-1}\|^2}$ for some constant $C$, and we define $\bg_0=\boldsymbol{0}$. 
In Algorithm~\ref{alg:template_no_epochs}, set $\A$ to be $\A_\text{static}$ that is reset every $T$ rounds.
Let $\delta>0$ be an arbitrary number. Set $D=\delta/T$, $T=\min(\lceil\frac{(C\delta^2 \sqrt{H}N)^{2/5}}{(F(\bx_0)-F^\star)^{2/5}}\rceil, \frac{N}{2})$, and $K=\lfloor \frac{N}{T}\rfloor$.
Finally, suppose that $F$ is $H$-smooth and that the gradient oracle is deterministic (that is, $\bg_n=\nabla F(\bw_n)$). Then we have:


\begin{align*}
\iftoggle{icml}{ &}
\E\left[\frac{1}{K}\sum_{t=1}^K \|\nabla F(\overline{\bw}^k)\|_\delta \right]
\iftoggle{icml}{}{&}
\le \frac{2CG_1}{N}+\frac{2(F(\bx_0)-F^\star)}{\delta N}
\\&
\quad
+\max\left(6\frac{(CH)^{2/5}(F(\bx_0)-F^\star)^{3/5}}{\delta^{1/5}N^{3/5}}, \frac{17 C \delta \sqrt{H}}{N^{3/2}}\right)~.
\end{align*}
\end{Theorem}
Note that the expectation here encompasses only the randomness in the choice of $s^k_t$, because the gradient oracle is assumed to be deterministic.
Theorem~\ref{thm:detsmooth} finds a $(\delta,\epsilon)$ stationary point in $O(\epsilon^{-5/3}\delta^{-1/3})$ iteratations. Thus, by setting $\delta = \epsilon/H$, Proposition~\ref{prop:nonsmoothtosmooth} shows we can find a $(0,\epsilon)$ stationary point  in $O(\epsilon^{-2})$ iterations, which matches the standard optimal rate~\citep{carmon2021lower}.

\begin{proof}
First, observe that for all $k,t$, $\|\overline{\bw}^k-\bw^k_t\|\le \delta$. This holds for precisely the same reason that it holds in Corollary~\ref{thm:nonsmoothogd}.


Next, observe that for $k=1$ we have
\begin{align*}
 R_T(\bu^k)
 &\le CD\sqrt{\sum_{t=1}^T \|\bg^k_t-\bg^k_{t-1}\|^2}\\
 &\le CD\sqrt{G_1^2+\sum_{t=2}^T\|\nabla F(\bw^k_t) - \nabla F(\bw^k_{t-1})\|^2}\\
 &\le CD\sqrt{G_1^2 + \sum_{t=2}^T H^2\|\bw^k_t-\bw^k_{t-1}\|^2}\\
 &\le CD\sqrt{G_1^2 + 4H^2TD^2}
 \le CDG_1 + 2CD^2H\sqrt{T}~.
\end{align*}
Similarly, for $k>1$, we observe that
\begin{align*}
\iftoggle{icml}{&\sum_{t=1}^T}{\sum_{t=1}^T}
\|\bg^k_t-\bg^k_{t-1}\|^2
\iftoggle{icml}{\le \sum_{t=2}^T }{&\le\sum_{t=2}^T}
\|\nabla F(\bw^k_t) - \nabla F(\bw^k_{t-1})\|^2 
\iftoggle{icml}{ \\ &\qquad\qquad\qquad\qquad\quad +}{+} 
\|\nabla F(\bw^k_1)-\nabla F(\bw^{k-1}_T)\|^2\\
&\le H^2(\|\bw^k_1-\bw^{k-1}_{T}\|^2 + \sum_{t=2}^T \|\bw^k_t-\bw^k_{t-1}\|^2) 
\\&
\le 4  TH^2 D^2~.
\end{align*}
Thus, we have
\begin{align*}
 R_T(\bu^k)&\le CD\sqrt{\sum_{t=1}^T \|\bg^k_t-\bg^k_{t-1}\|^2}\\
 &\le CD\sqrt{ 4H^2TD^2}
 \\&
 \le  2CD^2H\sqrt{T}~.
\end{align*}
Now, applying Theorem~\ref{thm:nonsmooth} in concert with the above bounds on $R_T(\bu^k)$, we have
\begin{align*}
\iftoggle{icml}{&\E\left[\frac{1}{K}\sum_{k=1}^K \left\|\frac{1}{T}\sum_{t=1}^T \nabla F(\bw^k_t)\right\|\right]\\}{%
\E\left[\frac{1}{K}\sum_{k=1}^K \left\|\frac{1}{T}\sum_{t=1}^T \nabla F(\bw^k_t)\right\|\right]}
&\le \frac{2(F(\bx_0)-F^\star)}{DN} + \frac{2CDG_1 + 4CKD^2\sqrt{HT}}{DN}\\
&=\frac{2T(F(\bx_0)-F^\star)}{\delta N} +\frac{2CG_1}{N} +\frac{4C\delta \sqrt{H}}{T^{3/2}}\\
&\le \max\left(6\frac{(CH)^{2/5}(F(\bx_0)-F^\star)^{3/5}}{\delta^{1/5}N^{3/5}}, \frac{17 C \delta \sqrt{H}}{N^{3/2}}\right)\\
&\quad+ \frac{2CG_1}{N}+\frac{2(F(\bx_0)-F^\star)}{\delta N}~. 
\end{align*}
Recalling that $\|\bw_t^k -\overline{\bw}^k\|\le \delta$, the conclusion follows.
\qedhere
\end{proof}

\subsection{Better Results with Second-Order Smoothness}

When $F$ is $J$-second-order smooth (i.e., $\nabla^2F$ is $J$-Lipschitz) we can do even better. First, observe that by Theorem~\ref{thm:detsmooth}, if $F$ is $J$-second-order-smooth, then by Proposition~\ref{prop:nonsmoothtosecondsmooth}, the $O(\epsilon^{-5/3}\delta^{-1/3})$ iteration complexity of Theorem~\ref{thm:detsmooth} implies an $O(\epsilon^{-11/6})$ iteration complexity for finding $(0,\epsilon)$ stationary points by setting $\delta = \sqrt{\epsilon/J}$. This already improves upon the  $O(\epsilon^{-2})$ result for smooth losses, but we can improve still further.
The key idea is to generate more informative hints $\bh_t$. If we can make $\bh_t\approx \bg_t$, then by (\ref{eqn:omd_bound}), we can achieve smaller regret and so a better guarantee.

To do so, we abandon randomization: instead of choosing $s_n$ randomly, we just set $s_n=1/2$. This setting still allows $F(\bx_n)\approx F(\bx_{n-1}) + \langle \bg_n, \bDelta_n\rangle$ with very little error when $F$ is second-order-smooth. By inspecting the optimistic mirror descent update formula, we can identify an $\bh_t$ with $\|\bh_t-\bg_t\|\le O(1/\sqrt{N})$ using $O(\log(N))$ gradient queries. This more advanced online learning algorithm is presented in Algorithm~\ref{alg:omd} (full analysis in Appendix~\ref{sec:proof_goodhints}).

Overall, Algorithm~\ref{alg:omd}'s update has an ``implicit'' flavor:
\begin{align*}
\bDelta_{n} &= \Pi_{\|\bDelta\|\le D} \left[\bDelta_{n-1} - \frac{\bg_n}{2H}\right],\\
\bg_n &= \nabla F(\bx_{n-1} + \bDelta_n/2)~.
\end{align*}

\begin{algorithm}[t]
\caption{Optimistic Mirror Descent with Careful Hints}
\label{alg:omd}
\begin{algorithmic}
\STATE{\bfseries Input: } Learning rate $\eta$, number $Q$ ($Q$ will be $O(\log N)$), function $F$, horizon $T$, radius $D$
\STATE Receive initial iterate $\bx_0$
\STATE Set $\bDelta'_1=\boldsymbol{0}$
\FOR{$t=1\dots T$}
    \STATE Set $\bh^0_t = \nabla F(\bx_{t-1})$
    \FOR{$i=1\dots Q$}
        \STATE Set $\bh^{i}_t = \nabla F\left(\bx_{t-1}+ \frac{1}{2}\Pi_{\|\bDelta\|\le D} \left[\bDelta'_t - \eta \bh^{i-1}_t\right]\right)$
    \ENDFOR
    \STATE Set $\bh_t = \bh^{Q}_t$
    \STATE Output $\bDelta_t = \Pi_{\|\bDelta\|\le D} [\bDelta'_t - \eta \bh_t]$
    \STATE Receive $t$th gradient $\bg_t$
    \STATE Set $\bDelta'_{t+1} = \Pi_{\|\bDelta\|\le D} [\bDelta'_t - \eta \bg_t]$
\ENDFOR
\end{algorithmic}
\end{algorithm}

With this refined online algorithm, we can show the following convergence guarantee, whose proof is in Appendix~\ref{sec:proof_detsecondsmooth}.
\begin{Theorem}\label{thm:detsecondsmooth}
In Algorithm~\ref{alg:template_no_epochs},  assume that $\bg_n = \nabla F(\bw_n)$, 
and set $s_n=\frac{1}{2}$. Use 
Algorithm~\ref{alg:omd} restarted every $T$ rounds as $\A$.
Let $\delta>0$ an arbitrary number. Set $T=\min(\lceil\frac{(\delta^2(H+J\delta)N)^{1/3}}{(F(\bx_0)-F^\star)^{1/3}}\rceil,N/2)$ and $K=\lfloor\frac{N}{T}\rfloor$. 
In Algorithm~\ref{alg:omd}, set $\eta=1/2H$, $D=\delta/T$, and $Q=\lceil \log_2(\sqrt{NG/HD})\rceil$. Finally, suppose that $F$ is $J$-second-order-smooth. Then, the following facts hold:
\begin{enumerate}
\item For all $k,t$, $\|\overline{\bw}^k-\bw^k_t\|\le \delta$.
\item We have the inequality
\begin{align*}
\iftoggle{icml}{&\frac{1}{K}}{\frac{1}{K}}
\sum_{k=1}^K \left\|\frac{1}{T}\sum_{t=1}^T \nabla F(\bw^k_t)\right\|
\iftoggle{icml}{}{&}
\leq \frac{4G}{N} +\frac{2(F(\bx_0)-F^\star)}{N\delta }
\\
&\quad+ 3\frac{(H+J \delta)^{1/3}(F(\bx_0)-F^\star)^{2/3}}{\delta^{1/3}N^{2/3}}  + 10\frac{\delta(H+J\delta)}{N^2}~.
\end{align*}

\item With $\delta = \frac{H^{1/7} (F(\bx_0)-F(\bx_N))^{2/7}}{J^{3/7}N^{2/7}}$, we have
\begin{align*}
\frac{1}{K}&\sum_{t=1}^K \|\nabla F(\overline{\bw}^k)\|
\le O\left(\tfrac{J^{1/7}H^{2/7} (F(\bx_0)-F^\star)^{4/7}}{N^{4/7}}\right).
\end{align*}
Moreover, the total number of gradient queries consumed is $NQ=O(N\log(N))$
\end{enumerate}
\end{Theorem}

This result finds a $(\delta,\epsilon)$ stationary point in $\tilde O(\epsilon^{-3/2}\delta^{-1/2})$ iterations. Via Proposition~\ref{prop:nonsmoothtosecondsmooth}, this translates to $\tilde O(\epsilon^{-7/4})$ iterations for finding a $(0,\epsilon)$ stationary point, matching the best known rate (up to a logarithmic factor)~\citep{carmon2017convex}. Note that this may not be optimal: the best lower bound is $\Omega(\epsilon^{-12/7})$ \citep{carmon2021lower}. Intriguingly, our technique seems distinct from previous work, which usually relies on acceleration and detecting or exploiting negative curvature~\citep{carmon2017convex, agarwal2016finding, carmon2018accelerated, li2022restarted}. 


\section{Lower Bounds}\label{sec:mainlowerbounds}

In this section, we show that our $O(\epsilon^{-3}\delta^{-1})$ complexity achieved in Corollary~\ref{thm:nonsmoothogd} is tight.
We do this by a simple extension of the lower bound for stochastic \emph{smooth} non-convex optimization of \citet{arjevani2019lower}. We provide an informal statement and proof-sketch below. The formal result (Theorem~\ref{thm:formallower}) and proof is provided in Appendix~\ref{sec:app_lowerbounds}.

\begin{Theorem}[informal]\label{thm:lower_from_smooth}
There is a universal constant $C$ such that for any $\delta$, $\epsilon$, $\gap$ and $G\ge C\frac{\sqrt{\epsilon\gap}}{\sqrt{\delta}}$, for any first-order algorithm $\A$, there is a $G$-Lipschitz, $C^\infty$ function $F:\R^d\to \R$ for some $d$ with $F(0)-\inf_{\bx} F(\bx) \le \gap$ and a stochastic first-order gradient oracle  for $F$ whose outputs $\bg$ satisfy $\E[\|\bg\|^2]\le G^2$ such that such that $\A$ requires $\Omega(G^2\gap/\delta\epsilon^3)$ stochastic oracle queries to identify a point $\bx$ with $\E[\|\nabla F(\bx)\|_\delta]\le \epsilon$.
\end{Theorem}
\begin{proof}[Proof sketch]
The construction of \citet{arjevani2019lower} provides, for any $\sigma$ a function $F$ and stochastic oracle whose outputs have variance at most $\sigma^2$ such that $F$ is $H$-smooth, $O(\sqrt{H\gap})$-Lipschitz and $\A$ requires $\Omega(\sigma^2 H\gap/\epsilon^4)$ oracle queries to find a point $\bx$ with $\|\nabla F(\bx)\|\le 2\epsilon$. By setting $H=\frac{\epsilon}{\delta}$ and $\sigma=G/\sqrt{2}$, this becomes an $\sqrt{\epsilon\gap/\delta}$-Lipschitz function, and so is at most $G/\sqrt{2}$-Lipschitz. Thus, the second moment of the gradient oracle is at most $G^2/2 + G^2/2=G^2$. Further, the algorithm requires $\Omega(G^2\gap/\delta\epsilon^3)$ queries to find a point $\bx$ with $\|\nabla F(\bx)\|\le 2\epsilon$. Now, if $\|\nabla F(\bx)\|_\delta \le \epsilon$, then since $F$ is $H=\frac{\epsilon}{\delta}$-smooth, by Proposition~\ref{prop:nonsmoothtosmooth}, $\|\nabla F(\bx)\| \le \epsilon + \delta H = 2\epsilon$. Thus, we see that we need $\Omega(G^2\gap/\delta\epsilon^3)$ queries to find a point with $\|\nabla F(\bx)\|_\delta\le \epsilon$ as desired.
\end{proof}

\section{Conclusion}\label{sec:conclusion}

We have presented a new online-to-non-convex conversion technique that applies online learning algorithms to non-convex and non-smooth stochastic optimization. When used with online gradient descent, this achieves the optimal $\epsilon^{-3}\delta^{-1}$ complexity for finding $(\delta,\epsilon)$ stationary points.

These results suggest new directions for work in online learning. Much past work is motivated by the online-to-batch conversion relating \emph{static} regret to \emph{convex} optimization. We employ \emph{switching} regret for \emph{non-convex} optimization. More refined analysis may be possible via generalizations such as strongly adaptive or dynamic regret \citep{ daniely2015strongly, jun2017improved, zhang2018adaptive, jacobsen2022parameter, cutkosky2020parameter, lu2022adaptive, luo2022corralling, zhang2021dual, baby2022optimal, zhang2022adversarial}. Moreover, our analysis assumes perfect tuning of constants (e.g., $D, T, K$) for simplicity. In practice, we would prefer to adapt to unknown parameters, motivating new applications and problems for \emph{adaptive} online learning, which is already an area of active current investigation \citep[see, e.g.,][]{orabonaP15, hoeven2018many,cutkosky2018black,cutkosky2019combining,mhammedi2020lipschitz, chen2021impossible, sachs2022between, zhang2022parameter, wang2022adaptive}. We hope that this expertise can be applied in the non-convex setting as well. 

Finally, our results leave an important question unanswered: the current best-known algorithm for \emph{deterministic} non-smooth optimization still requires $O(\epsilon^{-3}\delta^{-1})$  iterations to  find a $(\delta, \epsilon)$-stationary point~\citep{zhang2020complexity}. We achieve this same result even in the stochastic case. Thus it is natural to wonder if the deterministic rate is tight.
For example, is the $O(\epsilon^{-3/2}\delta^{-1/2})$ complexity we achieve in the smooth setting also achievable in the non-smooth setting? Intriguingly, prior work \cite{kornowski2022complexity,jordan2022complexity} shows that randomization is necessary, even if the gradient oracle  itself is deterministic.

\section*{Acknowledgements}
The authors would like to thank Zijian Liu for identifying an error in the original proof of Theorem~\ref{thm:lower_from_smooth}.

Ashok Cutkosky is supported by the National Science Foundation grant CCF-2211718 as well as a Google gift.
Francesco Orabona is supported by the National Science Foundation under the grants no. 2022446
``Foundations of Data Science Institute'' and no. 2046096 ``CAREER: Parameter-free Optimization Algorithms
for Machine Learning''.

\bibliography{references, more_refs}
\bibliographystyle{icml2023}
\onecolumn
\appendix

\section{Proof of Proposition~\ref{prop:wellbehaved}}
\label{sec:proof_wellbehaved}

First, we state a technical lemma that will be used to prove Proposition~\ref{prop:wellbehaved}.
\begin{Lemma}\label{lem:annoyingtechnicalities}
Let $F:\R^d \to \R$ be locally Lipschitz. Then, $F$ is differentiable almost everywhere, is Lipschitz on all compact sets, and for all $\bv\in \R^d$, $\bx\mapsto \langle  \nabla F(\bx), \bv\rangle$ is integrable on all compact sets. Finally, for any compact measurable set $D\subset \R^d$, the vector $\bw=\int_D \! \nabla F(\bx)\, \mathrm{d}\bx$ is well-defined and the operator $\rho(\bv)= \int_{D}\! \langle \nabla F(\bx),\bv\rangle\, \mathrm{d}\bx$ is linear and equal to $ \langle \bw,\bv\rangle$.
\end{Lemma}
\begin{proof}
First, observe that since $F$ is locally Lipschitz, for every point $\bx\in \R^d$ with rational coordinates there is a neighborhood $U_{\bx}$ of $\bx$ on which $F$ is Lipschitz. Thus, by Rademacher's theorem, $F$ is differentiable almost everywhere in $U_{\bx}$. Since the set of points with rational coordiantes is dense in $\R^d$, $\R^d$ is equal to the countable union $\bigcup U_{\bx}$. Thus, since the set of points of non-differentiability of $F$ in $U_{\bx}$ is measure zero, the total set of points of non-differentability is a countable union of sets of measure zero and so must be measure zero. Thus $F$ is differentiable almost everywhere. This implies that $F$ is differentiable at $\bx + p\bu$ with probability 1.

Next, observe that for any compact set $S\subset \R^d$, for every point $\bx\in S$ with rational coordinates, $F$ is Lipschitz on some neighborhood $U_{\bx}$ containing $\bx$ with Lipschitz constant $G_{\bx}$. Since $S$ is compact, there is a finite set $\bx_1,\dots,\bx_K$ such that $S=\bigcup U_{\bx_i}$. Therefore, $F$ is $\max_i G_{\bx_i}$-Lipschitz on $S$ and so $F$ is Lipschitz on every compact set.


Now, for almost all $\bx$, for all $\bv$ we  have that the limit $\lim_{\delta\to 0} \frac{F(\bx+\delta \bv) - F(\bx)}{\delta}$ exists and is equal to $\langle \nabla F(\bx), \bv\rangle$ by definition of differentiability. Further, on any compact set $S$ we have that $\frac{|F(\bx+\delta \bv) - F(\bx)|}{\delta}\le L$ for some $L$ for all $\bx\in S$. Therefore, by the bounded convergence theorem~(see e.g., \citet[Theorem 1.4]{stein2009real}), we have that $\langle \nabla F(\bx), \bv\rangle$ is integrable on $S$.

Next, we prove the linearity of the operator $\rho$. Observe that for any vectors $\bv$ and $\bw$, and scalar $c$, by linearity of integration, we have
\begin{align*}
 \int_{D} \! \langle \nabla F(\bx),c \bv + \bw\rangle\, \mathrm{d} \bx 
 &= \int_{D}\! \langle \nabla F(\bx),c \bv\rangle + \langle \nabla F(\bx),\bw\rangle\, \mathrm{d}\bx\\
&=  \int_{D}\! c\langle \nabla F(\bx),\bv\rangle + \langle \nabla F(\bx),\bw\rangle\, \mathrm{d}\bx\\
&=  c\int_{D}\! \langle \nabla F(\bx),\bv\rangle \, \mathrm{d}\bx+ \int_D \!\langle \nabla F(\bx),\bw\rangle\, \mathrm{d}\bx~.
\end{align*}

For the remaining statement, given that $\R^d$ is finite dimensional, there must exist $\bw\in \R^d$ such that $\rho(\bv) = \langle \bw,\bv\rangle$ for all $\bv$. Further, $\bw$ is uniquely determined by $\langle \bw, \be_i\rangle$ for $i=1,\dots,d$  where $\be_i$ indicates the $i$th standard basis vector. Then, since $\nabla F(\bx)_i = \langle \nabla F(\bx), \be_i\rangle$ is integrable on compact sets, we have
\begin{align*}
\langle \bw, \be_i\rangle 
= \int_{D} \! \langle \nabla F(\bx),\be_i\rangle\, \mathrm{d}\bx 
= \int_{D}\! \nabla  F(\bx)_i\, \mathrm{d}\bx, 
\end{align*}
which  is  the definition of $\int_D\! \nabla F(\bx)\, \mathrm{d}\bx$ when the integral is defined.
\end{proof}

We can now prove the Proposition.
\begin{proof}[Proof of Proposition~\ref{prop:wellbehaved}]
Since $F$ is locally Lipschitz, by Proposition~\ref{lem:annoyingtechnicalities}, $F$ is Lipschitz on compact sets. Therefore, $F$ must be Lipschitz on the line segment connecting $\bx$ and $\by$. Thus the function $k(t)=F(\bx+t(\by-\bx))$ is absolutely continuous on $[0,1]$. As a result, $k'$ is integrable  on $[0,1]$ and $F(\by)-F(\bx) = k(1)-k(0) = \int_0^1 \! k'(t)\, \mathrm{d}t= \int_0^1 \! \langle \nabla F(\bx+t(\by-\bx)), \by-\bx\rangle\, \mathrm{d}t$ by the Fundamental Theorem of Calculus (see, e.g., \citet[Theorem 3.11]{stein2009real}).

Now, we tackle the case that $F$ is not differentiable everywhere. Notice that the last statement of the Proposition  is an immediate  consequence of Lipschitzness. So, we focus on showing the remaining parts.

Now, by Lemma~\ref{lem:annoyingtechnicalities}, we have that $F$ is differentiable almost everywhere. Further, $\bg_{\bx} = \E_{\bu}[\nabla F(\bx+ p \bu)]$ exists and satisfies for all $\bv\in \R^d$:
\begin{align*}
\langle \bg_{\bx}, \bv\rangle 
= \E_{\bu}[\langle \nabla F(\bx+p \bu), \bv\rangle]~.
\end{align*}
Notice also that
\begin{align*}
\E_{\bz,\bu}[\hat{\O}(\bx,(\bz,\bu))]
=\E_{\bu}\E_{\bz}[\O(\bx+p \bu, \bz)]
= \E_{\bu}[\nabla F(\bx+p \bu)]
= \bg_{\bx}~.
\end{align*}
So, it remains to show that $\hat F$ is differentiable and $\nabla \hat F(\bx) = \bg_{\bx}$.

Now,  let $\bx$ be  an arbitrary elements of $\R^d$ and let $\bv_1,\bv_2,\dots$ be any sequence of vectors such that  $\lim_{n\to \infty} \bv_n=0$ and $\|\bv_i\|\le p$ for all $i$. Then, since the ball of radius $2p$ centered at $\bx$ is compact, $F$ is $L$-Lipschitz inside this ball for some $L$. Then, we have 
\begin{align*}
&\lim_{n\to  \infty } \frac{\hat F(\bx+ \bv_n) - \hat F(\bx) - \langle \bg_{\bx}, \bv_n\rangle}{\|\bv_n\|} 
 = \lim_{n\to \infty} \E_{u}\left[\frac{F(\bx+\bv_n+ p \bu) - F(\bx+p \bu) - \langle \nabla F(\bx+p \bu), \bv_n\rangle }{\|\bv_n\|}\right]~.
\end{align*}
Now, observe $\frac{|F(\bx+\bv_i+ p \bu) - F(\bx+p \bu)|}{\|\bv_i\|}\le L$. Further, for all almost all $\bu$, $F$ is differentiable at $\bx+p \bu$ so that $\lim_{n\to \infty}  \frac{F(\bx+\bv_n+ p \bu) - F(\bx+p \bu) - \langle \nabla F(\bx+p \bu), \bv_n\rangle] }{\|\bv_n\|} = 0$ for almost all $\bu$. Thus, by the bounded convergence theorem, we have 
\begin{align*}
\lim_{n\to \infty } \frac{\hat F(\bx+ \bv_n) - \hat F(\bx) - \langle \bg_{\bx}, \bv_n\rangle}{\|\bv_n\|}& = 0~.
\end{align*}
which shows that $\bg_{\bx}=\nabla  \hat F(\bx)$.

Finally, observe that since $F$ is Lipschitz on compact sets, $\hat F$ must be also, and so by the first part of the proposition, $\hat F$ is well-behaved.
\end{proof}

\section{Analysis of (Optimistic) Online Gradient Descent}
\label{sec:online_algos}

Optimistic Online Gradient Descent (in its simplest form) is described by Algorithm~\ref{alg:gen_opt_md}.
Here we collect the standard analysis of the algorithm for completeness. None of this analysis is new, and more refined versions can be found in a variety of sources (e.g. \cite{chen2021impossible}).

\begin{algorithm}[h!]
   \caption{Optimistic Mirror Descent}
   \label{alg:gen_opt_md}
  \begin{algorithmic}
      \STATE{\bfseries Input: } Regularizer function $\phi$, domain $V$, time horizon $T$.
      \STATE $\hat \bw_1 =\boldsymbol{0}$
      \FOR{$t=1\dots T$}
      \STATE Generate ``hint'' $h_t$
      \STATE Set $\bw_t = \argmin_{\bx\in V} \ \langle \bh_t, \bx\rangle + \frac{1}{2}\|\bx- \hat \bw_t\|^2$
      \STATE Output $\bw_t$ and receive loss vector $\bg_t$
      \STATE Set $\hat \bw_{t+1} = \argmin_{\bx\in V} \ \langle \bg_t, \bx\rangle + \frac12 \|\bx- \hat \bw_t\|^2$
      \ENDFOR
   \end{algorithmic}
\end{algorithm}

We will analyze only a simple version of this algorithm, that is when $V$ is an $L_2$ ball of radius $D$ in some real Hilbert space (such as $\R^d$).
Then, Algorithm~\ref{alg:gen_opt_md} satisfies the following guarantee.
\begin{Proposition}
\label{prop:omd}
Let $V =\{\bx\ :\ \|\bx\|\le D\}\subset \mathcal{H}$ for some real Hilbert space $\mathcal{H}$.
Then, with for all $\bu\in V$, Algorithm~\ref{alg:gen_opt_md} ensures
\begin{align*}
\sum_{t=1}^T \langle \bg_t, \bw_t - \bu\rangle \le \frac{D^2}{2\eta} + \sum_{t=1}^T \frac{\eta}{2} \|\bg_t -\bh_t\|^2~.
\end{align*}
\end{Proposition}
\begin{proof}
Now, by \citet[Lemma 15]{chen2021impossible} instantiated with the squared Euclidean distance as Bregman divergence, we have
\begin{align*}
\langle \bg_t, \bw_t - \bu\rangle 
&\le \langle \bg_t - \bh_t, \bw_t - \hat \bw_{t+1}\rangle +\frac12\|\bu- \hat w_t\|^2 - \frac12\|\bu- \hat \bw_{t+1}\|^2 - \frac12\|\hat \bw_{t+1}- \bw_t\|^2 -\frac12\|\bw_t-  \hat \bw_t\|^2
\intertext{From Young inequality:}
 &\le \frac{\eta\|\bg_t-\bh_t\|^2}{2} + \frac{\|\bw_t - \hat \bw_{t+1}\|^2}{2\eta} +\frac12\|\bu- \hat \bw_t\|^2 - \frac12\|\bu- \hat \bw_{t+1}\|^2 - \frac12\|\hat \bw_{t+1}- \bw_t\|^2 \\
 &\qquad\qquad-\frac{\|\bw_t -  \hat \bw_t\|^2}{2\eta}\\
 &\le  \frac{\eta\|\bg_t-\bh_t\|^2}{2}  +\frac12\|\bu- \hat \bw_t\|^2 - \frac12\|\bu- \hat \bw_{t+1}\|^2~.
 \end{align*}
 Summing over $t$ and telescoping, we have
 \begin{align*}
 \sum_{t=1}^T \langle \bg_t, \bw_t -\bu\rangle
 &\le \frac12\|\bu- \hat \bw_1\|^2 - \frac12\|\bu- \hat \bw_{T+1}\|^2 + \sum_{t=1}^T \frac{\eta\|\bg_t-\bh_t\|^2}{2}\\
 &\le \frac{\|\bu-\hat \bw_1\|^2}{2\eta} +  \sum_{t=1}^T \frac{\eta\|\bg_t-\bh_t\|^2}{2}
 \le \frac{D^2}{2\eta} + \sum_{t=1}^T \frac{\eta\|\bg_t-\bh_t\|^2}{2}~. \qedhere
\end{align*}
\end{proof}

In the case that the hints $\bh_t$ are not present, the algorithm becomes online gradient descent~\citep{zinkevich2003online}. In this case, assuming $\E[\|\bg_t\|^2]\leq G^2$ and setting $\eta=\frac{D}{G\sqrt{T}}$ we obtain the $\E[R_T(\bu)]\leq D G \sqrt{T}$ for all $\bu$ such that $\|\bu\|\leq D$.

\section{Algorithm~\ref{alg:omd} and Regret Guarantee}
\label{sec:proof_goodhints}


\begin{Theorem}
\label{thm:goodhints}
Let $F$ be an $H$-smooth and $G$-Lipschitz function. Then, when $Q=\lceil \log_2(\sqrt{NG/HD})\rceil$, Algorithm~\ref{alg:omd} with $\bx_t = \bx_{t-1} + \frac{1}{2} \bDelta_t$ and $\bg_t = \nabla F(\bx_t)$ and $\eta \le \frac{1}{2H}$ ensures for all $\|\bu\|\le D$
\begin{align*}
\sum_{t=1}^T \langle \bg_t, \bDelta_t - \bu \rangle &\le \frac{HD^2}{2} + \frac{2G T D}{N}~.
\end{align*}
Furthermore, a total of at most $T\lceil \log_2(\sqrt{NG/HD})\rceil$ gradient evaluations are required.
\end{Theorem}
\begin{proof}
The count of gradient evaluations is immediate from inspection of the algorithm, so it remains only to prove the regret bound.

First, we observe that the choices of $\bDelta_t$ specified by Algorithm~\ref{alg:omd} correspond to the values of $\bw_t$ produced by Algorithm~\ref{alg:gen_opt_md} when $\psi(\bw) = \frac{1}{2\eta}\|\bw\|^2$. This can be verified by direct calculation (recalling that $D_{\psi}(\bx,\by) = \frac{\|\bx-\by\|^2}{2\eta})$.

Therefore, by Proposition~\ref{prop:omd}, we have
\begin{align}
\sum_{t=1}^T \langle \bg_t, \bDelta_t - \bu\rangle \le \frac{D^2}{2\eta} + \sum_{t=1}^T \frac{\eta}{2} \|\bg_t -\bh_t\|^2~.\label{eqn:omd_regret_pf}
\end{align}
So, our primary task is to show that $\|\bg_t-\bh_t\|$ is small. To this end, recall that $\bg_t = \nabla F(\bw_t) = \nabla F(\bx_{t-1} + \bDelta_t/2)$.

Now, we define $h^{M+1}_t = \nabla F\left(\bx_{t-1} + \frac{1}{2} \Pi_{\|\bDelta\|\le D}\left[\bDelta_t' - \eta \bh^{M}_t\right]\right)$ (which simply continues the recursive definition of $\bh^i_t$ in Algorithm~\ref{alg:omd} for one more step).
Then, we claim that for all $0\le i \le M$, $\|\bh^{i+1}_t - \bh^i_t\|\le \frac{1}{2^i}\|\bh^1_t- \bh^0_t\|$. We establish the claim by induction on $i$. First, for $i=0$ the claim holds by definition. Now suppose $\|\bh^{i}_t - \bh^{i-1}_t\|\le \frac{1}{2^{i-1}}\|\bh^1_t- \bh^0_t\|$ for some $i$. Then, we have
\begin{align*}
     \|\bh^{i+1}_t - \bh^i_t\|
     &\le \left\|\nabla F\left(\bx_{t-1} + \frac{1}{2}\Pi_{\|\bDelta\|\le D} \left[\bDelta'_t - \eta \bh^i_t\right]\right) - \nabla F\left(\bx_{t-1}+\frac{1}{2}\Pi_{\|\bDelta\|\le D} \left[\bDelta'_t - \eta \bh^{i-1}_t\right]\right)\right\|
     \intertext{Using the $H$-smoothness of $F$:}
     &\le \frac{H}{2}\left\|\Pi_{\|\bDelta\|\le D}\left[ \bDelta'_t - \eta \bh^i_t\right] - \Pi_{\|\bDelta\|\le D}\left[ \bDelta'_t - \eta \bh^{i-1}_t\right]\right\|
     \intertext{Using the fact that projection is a contraction:}
     &\le \frac{H\eta}{2} \left\| \bh^i_t- \bh^{i-1}_t\right\|
     \intertext{Using $\eta \le \frac{1}{H}$:}
     &= \frac{1}{2} \left\| \bh^i_t- \bh^{i-1}_t\right\|
     \intertext{From the induction assumption:}
     &\le \frac{1}{2^i}\|\bh^1_t- \bh^0_t\|~.
\end{align*}
So that the claim holds.

Now, since $\bh_t=\bh^Q_t$, we have $\bDelta_t =  \Pi_{\|\bDelta\|\le D}\left[\bDelta_t' - \eta \bh^{i-1}_t\right]$. Therefore $g_t = \nabla F(\bx_t) = \nabla F(\bx_{t-1} + \bDelta_t/2) = \bh^{Q+1}_t$. Thus,
\begin{align*}
    \|\bg_t - \bh^Q_t\|&= \|\bh^{Q+1}_t - \bh^Q_t\|
    \le\frac{1}{2^Q} \|\bh^1_t - \bh^0_t\|
    \le \frac{2G}{2^Q},
\end{align*}
where in the last inequality we used the fact that $F$ is $G$-Lipschitz.
So, for $Q=\lceil \log_2(\sqrt{NG/HD})\rceil$, we have $\|\bg_t - \bh^Q_t\|\le \frac{2G\sqrt{HD}}{\sqrt{NG}}$ for all $t$. The result now follows by substituting into equation (\ref{eqn:omd_regret_pf}). 
\end{proof}

\section{Proof of Theorem~\ref{thm:detsecondsmooth}}
\label{sec:proof_detsecondsmooth}

\begin{proof}[Proof of Theorem~\ref{thm:detsecondsmooth}]
Once more, the first part of the result is immediate from the fact that $\|\bDelta_n\|\le D$. So, we proceed to show the second part.

Define $\bnabla_n = \int_0^1 \! \nabla F(\bx_{n-1} + s\bDelta_n)\, \mathrm{d}s$. Then, we have
\begin{align*}
&\|\langle \bnabla_n-\bg_n,\bDelta_n\rangle\| \\
&= \left\|\left\langle \int_0^1\!\nabla F(\bx_{n-1} + s\bDelta_n)-\nabla F\left(\bx_{n-1} + \frac{1}{2}\bDelta_n\right)\, \mathrm{d}s,\ \bDelta_n\right\rangle\right\|\\
&\le D\left\|\int_0^1 \! \nabla F(\bx_{n-1} + s\bDelta_n)-\nabla F\left(\bx_{n-1} + \frac{1}{2}\bDelta_n\right)\, \mathrm{d}s\right\|\\
&= D\left\|\int_0^1 \! \left(\nabla F(\bx_{n-1} + s\bDelta_n)-\nabla F\left(\bx_{n-1} + \frac{1}{2}\bDelta_n\right)-\nabla^2F\left(\bx_{n-1}+\frac{1}{2}\bDelta_n\right)\bDelta_n(s-1/2)\right)\right.\\
&\qquad\qquad\left.+\nabla^2F(\bx_{n-1}+\frac{1}{2}\bDelta_n)\bDelta_n(s-1/2)\, \mathrm{d}s\right\| 
\intertext{(observing that $\int_0^1 \! s-1/2 \, \mathrm{d}s=0$)}
&= D\left\|\int_0^1 \! \nabla F(\bx_{n-1} + s\bDelta_n)-\nabla F(\bx_{n-1} + \frac{1}{2}\bDelta_n)-\nabla^2F\left(\bx_{n-1}+\frac{1}{2}\bDelta_n\right)\bDelta_n(s-1/2) \, \mathrm{d}s\right\|
\intertext{(using second-order smoothness)}
&\le D\int_0^1 \! \frac{J}{2}\|\bDelta_n\|^2(s-1/2)^2 \, \mathrm{d}s
\le \frac{JD^3}{48}~.
\end{align*}

In Theorem~\ref{thm:otnc}, set $\bu_{n}$ to be equal to $\bu^1$ for the first $T$ iterations, $\bu^2$ for the second $T$ iterations and so on. In other words, $\bu_n=\bu^{\lfloor n/T\rfloor+1}$ for $n=1, \dots,M$.
So, we have
\begin{align*}
F(\bx_M) - F(\bx_0)
&= R_T(\bu^1,\dots,\bu^K) + \sum_{n=1}^M \langle \bnabla_n - \bg_n,\bDelta_n\rangle + \sum_{n=1}^M \langle \bg_n, \bu_n\rangle\\
&\le R_T(\bu^1, \dots, \bu^K) + \frac{NJD^3}{48} +\sum_{n=1}^M \langle \bg_n, \bu_n\rangle~.
\end{align*}
Now, set $\bu_k = -D\frac{\sum_{t=1}^T \nabla F(\bw^k_t)}{\|\sum_{t=1}^T \nabla F(\bw^k_t)\|}$. Then, by Theorem~\ref{thm:goodhints}, we have that $R_T(\bu_k)\le \frac{HD^2}{2} + \frac{2TGD}{N}$. Therefore:
\begin{align*}
F(\bx_M)&\le F(\bx_0) +\frac{HD^2K}{2} + 2GD + \frac{MJD^3}{48} -DT\sum_{k=1}^K \left\|\frac{1}{T}\sum_{t=1}^T \nabla F(\bw^k_t)\right\|~.
\end{align*}
Hence, we obtain
\begin{align*}
\frac{1}{K}&\sum_{k=1}^K \left\|\frac{1}{T}\sum_{t=1}^T \nabla F(\bw^k_t)\right\|
\le \frac{F(\bx_0)-F(\bx_M)}{MD}  +\frac{HD}{2T} + \frac{2G}{M} + \frac{JD^2}{48}~.
\end{align*}
Note that from the choice of $K$ and $T$ we have $M=KT\geq N-T\geq N/2$.
So, using $D=\delta/T$, we have can upper bound the r.h.s. with
\begin{align*}
&\frac{2T(F(\bx_0)-F^\star)}{N\delta }  +\frac{H\delta }{2T^2} + \frac{4G}{N} + \frac{J\delta^2}{2T^2}
\intertext{and with $T=\min \left(\left\lceil\frac{(\delta^2(H+J\delta)N)^{1/3}}{(F(\bx_0)-F^\star)^{1/3}}\right\rceil,N/2\right)$ :}
&\le 3\frac{(H+J \delta)^{1/3}(F(\bx_0)-F^\star)^{2/3}}{\delta^{1/3}N^{2/3}} + \frac{4G}{N} +\frac{2(F(\bx_0)-F^\star)}{N\delta } + 10\frac{\delta(H+J\delta)}{N^2}~.
\end{align*}

Now, the third fact follows by observing that Proposition~\ref{prop:nonsmoothtosecondsmooth} implies that
\begin{align*}
\frac{1}{K}\sum_{k=1}^K \left\|\nabla F(\overline{\bw}^k)\right\|&\le \frac{1}{K}\sum_{k=1}^K \left\|\frac{1}{T}\sum_{t=1}^T \nabla F(\bw^k_t)\right\|+\frac{J\delta^2}{2}~.
\end{align*}
Now, substituting the specified value of $\delta$ completes the identity. Finally, the count of number of gradient evaluations is a direct calculation.
\end{proof}
\section{Proofs for Section \ref{sec:to_smooth}}\label{sec:proof_to_smooth}
\nonsmoothtosmooth*
\begin{proof}
Let $\bgamma\subset B(\bx,\delta)$ with $\bx=\frac{1}{|\bgamma|}\sum_{\by\in \bgamma} \by$. By $H$-smoothness, for all $\by\in \bgamma$, $\|\nabla F(\by)-\nabla F(\bx)\|\le H\|\by-\bx\|\le H\delta$. Therefore, we have
\begin{align*}
\left\|\frac{1}{|\bgamma|}\sum_{\by\in \bgamma} \nabla F(\by)\right\|
& =\left\|\nabla F(\bx)+\frac{1}{|\bgamma|}\sum_{\by\in \bgamma} (\nabla F(\by)-\nabla F(\bx))\right\|\\
&\ge \|\nabla F(\bx)\| - H\delta~.
\end{align*}
Now, since $\|\nabla F(\bx)\|_\delta\le \epsilon$, for any $p>0$, there is a set $\bgamma$ such that $\left\|\frac{1}{|\bgamma|}\sum_{\by\in \bgamma} \nabla F(\by)\right\|\le \epsilon+p$. Thus, $\|\nabla F(\bx)\|\le \epsilon + H\delta +p$ for any $p>0$, which implies $\|\nabla F(\bx)\|\le \epsilon + H\delta$.
\end{proof}

\nonsmoothtosecondsmooth*
\begin{proof}
The proof is similar to that of Proposition~\ref{prop:nonsmoothtosmooth}. 
Let $\bgamma\subset B(\bx,\delta)$ with $\bx=\frac{1}{|\bgamma|}\sum_{\by\in \bgamma} \by$. By $J$-second-order-smoothness, for all $\by\in \bgamma$, we have
\begin{align*}
\|\nabla F(\by) - \nabla F(\bx) - \nabla^2 F(\bx)(\by-\bx)\|
&=\left\|\int_0^1 \! (\nabla^2F(\bx+t(\by-\bx))-\nabla^2F(\bx))(\by-\bx)\, \mathrm{d}t\right\|\\
&\le \int_0^1 \! t J\|\by-\bx\|^2\, \mathrm{d}t
=\frac{J\|\by-\bx\|^2}{2}
\leq \frac{J\delta^2}{2}~.
\end{align*}
Further, since $\frac{1}{|\bgamma|}\sum_{\by\in \bgamma} \by = \bx$, we have $\frac{1}{|\bgamma|}\sum_{\by\in\bgamma}\nabla^2F(\bx)(\by-\bx)=0$. Therefore, we have
\begin{align*}
\left\|\frac{1}{|\bgamma|}\sum_{\by\in \bgamma} \nabla F(\by)\right\|
&=\left\|\nabla F(\bx)+\frac{1}{|\bgamma|}\sum_{\by\in \bgamma} (\nabla F(\by)-\nabla F(\bx) - \nabla^2F(\bx)(\by-\bx))\right\|\\
&\ge \|\nabla F(\bx)\| - \frac{J\delta^2}{2}~.
\end{align*}
Now, since $\|\nabla F(x)\|_\delta\le \epsilon$, for any $p>0$, there is a set $\bgamma$ such that $\left\|\frac{1}{|\bgamma|}\sum_{\by\in \bgamma} \nabla F(\by)\right\|\le \epsilon+p$. Thus, $\|\nabla F(\bx)\|\le \epsilon + \frac{J}{2}\delta^2 +p$ for any $p>0$, which implies $\|\nabla F(\bx)\|\le \epsilon + \frac{J}{2}\delta^2$.
\end{proof}

\section{Lower Bounds}\label{sec:app_lowerbounds}

Our lower bounds are constructed via a mild alteration to the arguments of \citet{arjevani2019lower} for lower bounds on finding $(0,\epsilon)$-stationary points of \emph{smooth} functions with a stochastic gradient oracle.  At a high level, we show that since a $\delta,\epsilon$-stationary point of an $H$-smooth loss is also a $(0,H\delta+\epsilon)$-stationary point, a lower bound on the complexity of the latter implies a lower bound on the of complexity of the former. The lower bound of \citet{arjevani2019lower} is proved by constructing a distribution over ``hard'' functions such that no algorithm can quickly find a $(0,\epsilon)$-stationary point of a random selected function. Unfortunately, these ``hard'' functions are not Lipschitz. Fortunately, they take the form $F(\bx) + \eta \|\bx\|^2$ where $F$ is Lipschitz and smooth so that the ``non-Lipschitz''  part is solely contained in the quadratic term.  We show that one can replace the quadratic term $\|\bx\|^2$ with a Lipschitz function that is quadratic for sufficiently small $\bx$ but proportional  to $\|\bx\|$ for larger values. Our proof consists of carefully reproducing the argument of \citet{arjevani2019lower} to show that this modification does not cause any problems. We emphasize that almost all of this development can be found with more detail in \citet{arjevani2019lower}. We merely restate here the minimum results required to verify our modification to their construction.

\subsection{Definitions and Results from \citet{arjevani2019lower}}

A randomized first-order algorithm is a distribution $P_S$ supported on a set $S$ and a sequence of measurable mappings $A_i(s, \bg_1,\dots,\bg_{i-1})\to  \R^d$ with $s\in S$ and $\bg_i\in \R^d$. Given a stochastic gradient oracle $\O:\R^d\times Z\to  \R^d$, a distribution $P_Z$ supported on $Z$ and an i.i.d. sample $(\bz_1,\dots,\bz_n)\sim P_Z$, we define the \emph{iterates} of $A$  recursively by:
\begin{align*}
    \bx_1 &= A_1(s)\\
    \bx_i&= A_i(s, \O(\bx_1, \bz_1),\O(\bx_2,\bz_2),\dots,\O(\bx_{i-1},\bz_{i-1}))~.
\end{align*}
So, $\bx_i$ is a function of $s$ and $\bz_1,\dots,\bz_{i-1}$. We define $\Arand$ to be the set of such sequences of mappings.

Now, in the notation of \citet{arjevani2019lower}, we define the ``progress function''
\begin{align*}
\text{prog}_c(\bx) = \max \{i\ :\ |\bx_i| \ge c\}~.
\end{align*}

Further, a stochastic gradient oracle $\O$ can be called a probability-$p$ zero-chain if $\text{prog}_0(\O(\bx, \bz)) = \text{prog}_{1/4}(\bx) + 1$ for all $\bx$ with probability at least $1-p$, and $\text{prog}_0(\O(\bx, \bz))\le \text{prog}_{1/4}(\bx)+1$ with probability 1.

Next, let $F_T:\R^T\to \R$ be the function defined by Lemma 2 of \citet{arjevani2019lower}. Restating their Lemma, this function satisfies:
\begin{Lemma}[Lemma 2 of \cite{arjevani2019lower}]\label{lem:FTdef}
There exists a function $F_T:\R^T\to \R$ satisfies that satisfies:
\begin{enumerate}
\item $F_T(0)=0$ and $\inf F_T(\bx)\ge - \gap_0 T$ for $\gap_0 = 12$.
\item $\nabla F_T(\bx)$ is $H_0$-Lipschitz, with $H_0 = 152$.
\item For all $\bx$, $\|\nabla F_T(\bx)\|_\infty\le G_0$ with $G_0=23$
\item For all $\bx$, $\text{prog}_0(\nabla F_T(\bx))\le \text{prog}_{1/2}(\bx)+1$.
\item If $\text{prog}_1(\bx)<T$, then $\|\nabla F_T(\bx)\|\ge |\nabla F_T(\bx)_{\text{prog}_1(\bx)+1}|\ge 1$.
\end{enumerate}
\end{Lemma}

We associate with this function $F_T$ the stochastic gradient oracle $O_T(\bx, z):\R^T\times \{0,1\}\to \R^d$ where $z$ is  Bernoulli$(p)$:
\begin{align*}
 \O_T(\bx ,\bz)_i=\left\{\begin{array}{ll}
\nabla  F_T(\bx)_i, & \text{ if } i\ne \text{prog}_{1/4}(\bx)\\
\frac{\bz \nabla  F_T(\bx)_i }{p }, & \text{ if } i= \text{prog}_{1/4}(\bx)
\end{array}\right.
\end{align*}
It is clear that $\E_z[O_T(\bx ,\bz)] = \nabla F_T(\bx)$.

This construction is so far identical to that in \citet{arjevani2019lower}, and so we have by their Lemma 3:
\begin{Lemma}[Lemma 3 of \citet{arjevani2019lower}]\label{lem:zerochain}
$\O_T$ is a probability-$p$ zero chain, has variance $\E[\|\O_T(\bx, \bz)- \nabla F_T(\bx)\|^2]\le G_0^2/p$, and $\|\O_T(\bx,z)\|\le  \frac{G_0}{p} + G_0\sqrt{T}$.
\end{Lemma}
\begin{proof}
The probability $p$ zero-chain and variance statements are directly from \citet{arjevani2019lower}. For the bound on $\|\O_T\|$, observe that $\O_T(\bx,\bz)=\nabla F_T(\bx)$ in all but one coordinate. In that one coordinate, $\O_T(\bx,\bz)$ is at most $\frac{\|\nabla F_T(\bx)\|_{\infty}}{p} = \frac{G_0}{p}$. Thus, the bound follows by triangle inequality.
\end{proof}

Next, for any matrix $U\in \R^{d\times T}$ with  orthonormal columns, we define $F_{T, U}:\R^d\to \R$ by:
\begin{align*}
F_{T,U}(\bx) = F_T(U^\top \bx)~.
\end{align*}
The associated stochastic gradient oracle is:
\begin{align*}
\O_{T,U}(\bx, \bz) = U\O_T(U^\top \bx, \bz)~.
\end{align*}

Now, we restate  Lemma 5 of \citet{arjevani2019lower}:
\begin{Lemma}[Lemma 5 of \citet{arjevani2019lower}]\label{lem:rotation}
Let $R>0$ and suppose $A\in \Arand$ is such that $A$  produces iterates $\bx_t$ with $\|\bx_t\|\le  R$. Let $d\ge \left\lceil 18 \frac{R^2T}{p}\log \frac{2T^2}{p c} \right\rceil$
Suppose $U$ is chosen uniformly at random from the set of $d\times T$ matrices with orthonormal columns. Let $\O$ be an probability-$p$ zero chain and let $\O_U(\bx,\bz) = U\O(U^\top \bx, \bz)$. Let $\bx_1,\bx_2,\dots$ be the iterates of $A$ when provided the stochastic gradient oracle $\O_{U}$. Then with probability at least $1-c$ (over the randomness of $U$, the oracle, and also the seed $s$ of $A$):
\begin{align*}
\text{prog}_{1/4}(U^\top \bx_t)< T ~~\text{for all }t\le \frac{T- \log(2/c)}{2p}~.
\end{align*}
\end{Lemma}

\subsection{Defining the ``Hard'' Instance}

Now, we for the first time diverge from the construction of \citet{arjevani2019lower} (albeit only slightly). Their construction uses a ``shrinking function'' $\rho_{R,d}:\R^d\to \R^d$ given by $\rho_{R,d}(\bx) = \frac{\bx}{\sqrt{1+\|\bx\|^2/R^2}}$ as well as an additional quadratic term to overcome the limitation of bounded iterates. We cannot tolerate the non-Lipschitz quadratic term, so we replace it with a Lipschitz version $q_{B,d}(\bx) = \bx^\top \rho_{B,d}(\bx)=\frac{\|\bx\|^2}{\sqrt{1+\|\bx\|^2/R^2}}$. Intuitively, $q_{B,d}$ behaves like $\|\bx\|^2$ for small enough $\bx$, but behaves like $R\|\bx\|$ for large $\|\bx\|$. Overall, we consider the function:
\begin{align*}
\hat F_{T,U}(\bx) &= F_{T,U}(\rho_{R,d}(\bx)) + \eta q_{B,d}(\bx)\\
&= F_{T}(U^\top\rho_{R,d}(\bx)) + \eta q_{B,d}(\bx)~.
\end{align*}
The stochastic gradient oracle associated with $\hat F_{T,U}(\bx)$ is
\begin{align*}
\widehat \O_{T,U}(\bx, \bz) &= J[\rho_{R,d}](\bx)^\top U \O_T(U^\top \rho_{R,d}(\bx),\bz) + \eta \nabla q_{B,d}(\bx)~.
\end{align*}
where $J[f](\bx)$ indicates the Jacobian of the function $f$ evaluated at $\bx$.

A description of the relevant properties of $q_B$ is provided in Section~\ref{sec:qfunc}. 

Next we produce a variant on  Lemma 6 from \citet{arjevani2019lower}. This is the most delicate part of our alteration, although the proof is still almost identical to  that of \citet{arjevani2019lower}.

\begin{Lemma}[variant on Lemma 6 of \cite{arjevani2019lower}]\label{lem:hardbound}
Let $R=B=60G_0\sqrt{T}$. Let $\eta=1/10$ and $c\in(0,1)$ and $p\in(0,1)$ and $T\in \mathbb{N}$. Set $d= \lceil 18\frac{R^2 T}{p}\log\frac{2T^2}{pc}\rceil$ and let $U$ be sampled uniformly from the set of $d\times T$ matrices with orthonormal columns. Define $\hat F_{T,U}$ and $\hat \O_{U,T}$ as above. Suppose $A\in \Arand$ and let $\bx_1,\bx_2,\dots$ be the iterates of $A$ when provided with $\hat \O_{U,T}$ as input. Then with probability at least $1-c$:
\begin{align*}
\|\nabla \hat F_{T,U}(\bx_t)\|\ge 1/2~~~\text{ for all }t\le \frac{T-\log(2/c)}{2p}~.
\end{align*}
\end{Lemma}
\begin{proof}
Define $\by_i=\rho_{R,d}(\bx_i)$. Recall the defniition:
\begin{align*}
    \O_{T,U}(\by, \bz) = U \O_T(U^\top \by, \bz)
\end{align*}
Then observe that $ \widehat \O_{T,U}(\bx, \bz)$ can be computed from $\bx$ and $\O_{T,U}(\by, \bz)$:
\begin{align*}
    \widehat \O_{T,U}(\bx, \bz) &= J[\rho_{R,d}](\bx)^\top  \O_{T,U}(\by, \bz) + \eta \nabla q_{B,d}(\bx)
\end{align*}
Therefore, we may consider the $\by$ to be the iterates of some \emph{different} algorithm $A^y\in \Arand$ applied to the oracle $\O_{T,U}(\by, \bz)$ ($A^y$ computes $\widehat \O_{T,U}(\bx, \bz)$ from $ \O_{T,U}(\by, \bz)$, and then applies the original algorithm $A$ to get $\bx$ and then $\rho_{B,d}$ to get $\by$).

Furthermore, it is clear from the definition of $\rho_{B,d}$ that $\|\by_i\|=\|\rho_{B,d}(\bx_i)\|\le R$ for all $i$. 

All together, this implies that $A^y$ satisfies the conditions of Lemma~\ref{lem:rotation}, and so we have that with probability at least $1-c$:
\begin{align*}
    \prog_{1/4}(U^\top \by_t)\le T\text{ for all }t\le \frac{T-\log(2/c)}{2p}
\end{align*}

Now, our goal is to show that $\|\nabla F(\bx_i)\| \ge 1/2$. We consider two cases, either $\|\bx_i\|>R/2$ or not.

First, observe that for $\bx_i$ with $\|\bx_i\|>R/2$, we have:
\begin{align*}
\|\nabla \hat F_{T,U}(\bx_i)\|&\ge \eta \|\nabla q_{B,d}(\bx_i)\| - \|J[\rho_{R,d}](\bx_i)\|_{\text{op}}\|\nabla \hat F( U^\top \by_i)\|\\
\intertext{using the fact that $\|J[\rho_{R,d}](\bx_i)\|_{\text{op}}\|\le 1$ (see \cite{arjevani2019lower} Lemma 15) as well as Proposition~\ref{prop:q_bounds} part 3:}
&\ge \eta \frac{\|\bx_i\|}{\sqrt{1+\|\bx_i\|^2/B^2}} - G_0\sqrt{T}\\
\intertext{using $B=R$ and $\|\bx_i\|>R/2$:}
&\ge \frac{\eta B}{\sqrt{5}} - G_0\sqrt{T}\\
&\ge \frac{\eta B}{3} - G_0\sqrt{T}
\intertext{Recalling $B=R=60G_0\sqrt{T}$ and $\eta = 1/10$:}
&= G_0\sqrt{T}
\intertext{Recalling $G_0=23$:}
&\ge 1/2~.
\end{align*}

Alternatively, suppose $\|\bx_i\|\le  R/2$. Then, let us set $j=\prog_1(U^\top \by_i)+1\le T$ (the inequality follows since $\prog_1\le \prog_{1/4}$). Then, if $\bu^j$ indicates the $j$th row of $u$, Lemma~\ref{lem:FTdef} implies:
\begin{align*}
|\langle \bu^j, \by_i\rangle|&< 1,\\
|\langle \bu^j, \nabla F_{T,U}(\by_i)\rangle|&\ge 1~.
\end{align*}
Next, by direct calculation we have $J[\rho_R](\bx_i)  = \frac{I - \rho_R(\bx_i)\rho_R(\bx_i)^\top/R^2}{\sqrt{1+\|\bx_i\|^2/R^2}} = \frac{I - \by_i\by_i^\top/R^2}{\sqrt{1+\|\bx_i\|^2/R^2}} $ so that:
\begin{align*}
\langle \bu^j, \nabla \hat F_{T,U}(\bx_i)\rangle &= \langle \bu^j, J[\rho_R](\bx_i)^\top \nabla F_{T,U}(\by_i)\rangle + \eta \langle \bu^j, \nabla q_B(\bx_i)\rangle\\
&=\frac{\langle \bu^j, \nabla F_{T,U}(\by_i)\rangle}{\sqrt{1+\|\bx_i\|^2/R^2}} - \frac{\langle \bu^j, \by_i\rangle \langle \by_i, \nabla F_{T,U}(\by_i)\rangle/R^2}{\sqrt{1+\|\bx_i\|^2/R^2}} + \eta \langle \bu^j, \nabla q_B(\bx_i)\rangle ~.
\end{align*}
Now, by Proposition~\ref{prop:q_bounds}, we have $\nabla q_B(\bx_i) = \left(2 - \frac{\|\by_i\|^2}{B^2}\right)\by_i$. So, (recalling $R=B$):
\begin{align*}
\langle \bu^j, \nabla \hat F_{T,U}(\bx_i)\rangle &=\frac{\langle \bu^j, \nabla F_{T,U}(\by_i)\rangle}{\sqrt{1+\|\bx_i\|^2/R^2}} - \frac{\langle \bu^j, \by_i\rangle \langle \by_i, \nabla F_{T,U}(\by_i)\rangle/R^2}{\sqrt{1+\|\bx_i\|^2/R^2}} + \eta \left(2 - \frac{\|\by_i\|^2}{R^2}\right) \langle \bu^j, \by_i\rangle
\intertext{Observing that $\|y_i\|\le \|x_i\|\le R/2$ and $|\langle \bu^j, \by_i\rangle|<1$:}
|\langle \bu^j, \nabla \hat F_{T,U}(\bx_i)\rangle|& \ge \frac{2|\langle \bu^j, \nabla F_{T,U}(\by_i)\rangle|}{\sqrt{5}} - \frac{\|\nabla F_{T,U}(\by_i)\|}{2R} -2\eta
\intertext{Using $|\langle \bu^j, \nabla F_{T,U}(\by_i)\rangle|\ge 1$ and $\|\nabla F(\by_i)\|\le G_0\sqrt{T}$:}
&\ge \frac{2}{\sqrt{5}} - \frac{G_0\sqrt{T}}{2R} - 2\eta
\intertext{With $R=60G_0\sqrt{T}$ and $\eta = 1/10$:}
&=\frac{2}{\sqrt{5}} - \frac{1}{120} - \frac{1}{5}\\
&>1/2~. \qedhere
\end{align*}
\end{proof}

Next, we observe some basic facts about the function $\hat F_{T,U}$:
\begin{Lemma}[variation on Lemma 7 in \citet{arjevani2019lower}]\label{lemma:finalFprops}
With the settings of $R,B,\eta$ in  Lemma~\ref{lem:hardbound}, the function $\hat F_{T,U}$ satisfies:
\begin{enumerate}
\item $\hat F_{T,U}(0) - \inf \hat F_{T,U}(\bx) \le \gap_0 T=12 T$
\item $\|\nabla \hat F_{T,U}(\bx)\|\le G_0\sqrt{T}+ 3\eta B\le 437\sqrt{T}$ for all $\bx$.
\item $\nabla \hat F_U(\bx)$ is $H_0 +3+ 8\eta\le 156$-Lipschitz.
\item $\|\widehat \O_{T,U}(\bx, \bz)\|\le \frac{G_0}{p} + G_0\sqrt{T}+ 3\eta B\le \frac{23}{p} + 437\sqrt{T}$ with probability 1.
\item $\widehat \O_{T,U}$ has variance at most $\frac{G_0^2}{p}\le \frac{23^2}{p}$
\end{enumerate}
\end{Lemma}
\begin{proof}
\begin{enumerate}
\item This property follows immediately from the fact that $F_T(0)- \inf F_T(\bx) \le \gap_0 T$.
\item Since $\rho_R$ is 1-Lipschitz for all $R$ and $q_{B}$ is $ 3B$-Lipschitz (see Proposition~\ref{prop:q_bounds}), $\hat F_{T,U}(\bx)$ is $G_0\sqrt{T}+ 3\eta B$-Lipschitz, where $G_0=23$ and $B=60 G_0\sqrt{T}$ and $\eta=1/10$
\item By assumption, $R\ge \max\left( H_0, 1\right)$. Thus, by \citet[Lemma 16]{arjevani2019lower}, $\nabla F_{T}(\rho_R(\bx))$ is $H_0 + 3$-Lipschitz and so $\nabla \hat F_{T,U}$ is $H_0 +3+ 8\eta$-Lipschitz by Proposition~\ref{prop:q_bounds}.
\item Since $\| \O_T\|\le \frac{G_0}{p} + G_0\sqrt{T}$, and $J[\rho_R](\bx)^\top U$ has operator norm at most 1, the bound follows.
\item Just as in the previous part, since $\O_T$ has variance $G_0^2/p$ and $J[\rho_R](\bx)^\top U$ has operator norm at most 1, the bound follows.
\end{enumerate}
\end{proof}

Now, we are finally in a position to prove:
\begin{restatable}{Theorem}{thmsmoothlower}\label{thm:smoothlowerbound}
Given any $\gap$, $H$, $\epsilon$, and $\sigma$ such that $\frac{\gap H}{48\cdot 156 \epsilon^2}\ge 1$, there exists a distribution over functions $F$ and stochastic first-order oracles $\O$ such that with probability 1, $F$ is $H$-smooth, $F(0)-\inf F(\bx)\le \gap$, $F$ is $11 \sqrt{H \gap}$-Lipschitz and $\O$ has variance $\sigma^2$, and for any algorithm in $\Arand$, with probability at least $1-c$, when provided a randomly selected $\O$, $A$ requires at least $\Omega\left(\frac{\gap H \sigma^2}{\epsilon^4}\right)$
iterations to output a point $\bx$ with $\E[\|\nabla F(\bx)\|]\le \epsilon$.
\end{restatable}
\begin{proof}
From Lemma~\ref{lemma:finalFprops} and Lemma~\ref{lem:hardbound}, we have a distribution over functions $F$ and first-order oracles such that with probability 1, $F$ is $437\sqrt{T}$ Lipschitz, $F$ is $ 156$-smoooth, $F(0) - \inf F(\bx) \le 12T$, $\O$ has variance at most $G_0^2/p=23^2/p$, and with probability at least $1-c$,
\begin{align*}
\|\nabla F(\bx_t)\|\ge 1/2~~~\text{ for all }t\le \frac{T-\log(2/c)}{2p}~.
\end{align*}
Now, set $\lambda=\frac{156}{H}\cdot 2\epsilon$, $T=\lfloor \frac{156\gap}{12 H \lambda^2}\rfloor = \lfloor \frac{\gap H}{48\cdot 156 \epsilon^2}\rfloor\ge \frac{\gap H}{96\cdot 156 \epsilon^2}$ and $p=\min\left(\frac{23^2H^2\lambda^2}{ 156^2 \sigma^2},1\right)$. Then, define
\begin{align*}
F_\lambda(\bx) = \frac{H \lambda^2}{156} F(\bx/\lambda)~.
\end{align*}
Then $F_\lambda$ is $\frac{H\lambda^2}{H_0}\cdot\frac{1}{\lambda^2}\cdot 156=H$- smooth, $F_\lambda (0)-\inf_{\bx}F_\lambda(\bx) \le 12 \cdot T \cdot \frac{H\lambda^2}{156}\le \gap$, and $F_\lambda$ is $437\sqrt{T} \frac{H\lambda}{156}\le 11\sqrt{H \gap}$- Lipschitz. We can construct an oracle $\O_\lambda$ from $\O$ by:
\begin{align*}
\O_\lambda(\bx, \bz) = \frac{H \lambda}{156} \O(\bx/\lambda,\bz)~.
\end{align*}
so that if $p<1$, we have:
\begin{align*}
\E[\|\O_\lambda(\bx,\bz) - \nabla F_\lambda(\bx)\|^2]&\le \frac{H^2 \lambda^2}{156^2} \cdot \frac{23^2}{p}
= \sigma^2~.
\end{align*}
Alternatively, if $p=1$, clearly the variance is 0.

Further, since an oracle for $F_\lambda$ can be constructed from the oracle for $F$, if we run $A$ on $F_\lambda$, with probability at least $1-c$,
\begin{align*}
\|\nabla F_\lambda(\bx_t)\|= \frac{H\lambda}{156} \|\nabla F(\bx_t)\|\ge \epsilon~~~\text{ for all }t\le \frac{T-\log(2/c)}{2p}~.
\end{align*}
Finally, we calculate:
\begin{align*}
    \frac{T-\log(2/c)}{2p} \ge 1.5\cdot 10^{-8}\cdot \frac{\sigma^2 \gamma H}{\epsilon^4} - 3\cdot 10^{-4} \frac{\sigma^2 \log(2/c)}{\epsilon^2}~.
\end{align*}
Thus, there exists a constant $K$ and an $\epsilon_0$ such that for all $\epsilon<\epsilon_0$,
\[
\E[\|\nabla F_\lambda(\bx_t)\|] \ge \epsilon~~~~\text{for all }t\le K \frac{H\gap \sigma^2}{\epsilon^4}~. \qedhere
\]
\end{proof}

From this result, we have our main lower bound (the formal version of Theorem~\ref{thm:lower_from_smooth}):
\begin{restatable}{Theorem}{thmformallower}\label{thm:formallower}
For any $\delta$, $\epsilon$, $\gap$, $G\ge \frac{11\sqrt{2\epsilon\gap}}{\sqrt{\delta}}$, there is a distribution over $G$-Lipschitz $C^\infty$ functions $F$ with $F(0)-\inf F(\bx)\le \gap$ and stochastic gradient oracles $\O$ with $\E[\|\O(\bx, \bz)\|^2]\le G^2$ such that for any algorithm $A\in \Arand$, if $A$ is provided as input a randomly selected oracle $\O$, $A$ will require $\Omega(G^2\gap/\delta \epsilon^3)$ iterations to identify a point $x$ with $\E[\|\nabla F(\bx)\|_\delta]\le \epsilon$.
\end{restatable}
\begin{proof}
From Theorem~\ref{thm:smoothlowerbound}, for any $H$, $\epsilon'$, $\gap$, and $\sigma$ we have a distribution over $C^\infty$ functions $F$ and oracles $\O$ such that $F$ is $H$-smooth, $11\sqrt{H \gap}$-Lischitz and $F(0)-\inf F(\bx)\le \gap$ and $\O$ has variance $\sigma^2$ such that $A$ requires $\Omega(H \gap \sigma^2/\epsilon'^4)$ iterations to output a point $\bx$ such that $\E[\|\nabla F(\bx)\|]\le \epsilon'$. Set $\sigma = G/\sqrt{2}$, $H = \epsilon/\delta$ and $\epsilon'=2\epsilon$. Then, we see that $\O$ has variance $G^2/2$, and $F$ is $11\sqrt{H\gap} = \frac{11\sqrt{\epsilon \gap}}{\sqrt{\delta}}\le G/\sqrt{2}$-Lipschitz so that $\E[\|\O(\bx, \bz)\|^2]\le G^2$. Further by Proposition!\ref{prop:nonsmoothtosmooth}, if $\|\nabla F(\bx)\|_\delta \le \epsilon$, then $\|\nabla F(\bx)\|\le \epsilon + H\delta = 2\epsilon$. Therefore, since $A$ cannot output a point $\bx$ with $\E[\|\nabla F(\bx)\|]\le \epsilon'=2\epsilon$ in less than $\Omega( H\gap\sigma^2/\epsilon'^4)$ iterations, we see that $A$ also cannot output a point $\bx$ with $\E[\|\nabla F(\bx)\|_\delta]\le \epsilon$ in less than $\Omega( H\gap\sigma^2/\epsilon'^4)=\Omega(\gap G^2/\epsilon^3\delta)$ iterations.
\end{proof}

\subsection{Definition and Properties of $q_B$}\label{sec:qfunc}

Consider the function $q_{B,d}:\R^d \to \R$ defined by
\begin{align*}
q_{B,d}(\bx) = \frac{\|\bx\|^2}{\sqrt{1+\|\bx\|^2/B^2}} = \bx^\top\rho_{B,d}(\bx)~.
\end{align*}

This function has the following properties, all of which follow from direct calculuation:
\begin{Proposition}\label{prop:q_bounds}
$q_{B,d}$ satisfies:
\begin{enumerate}
\item 
\begin{align*}
\nabla q_{B,d}(\bx) = \frac{2\bx}{\sqrt{1+\|\bx\|^2/B^2}} - \frac{\bx\|\bx\|^2}{B^2(1+\|\bx\|^2/B^2)^{3/2}} = \left(2-\frac{\|\rho_B(\bx)\|^2}{B^2}\right) \rho_B(\bx)~.
\end{align*}
\item 
\begin{align*}
\nabla^2 q_{B,d}(\bx) &= \frac{1}{\sqrt{1+\|\bx\|^2/B^2}}\left(2I - \frac{3\bx\bx^\top}{B^2(1+\|\bx\|^2/B^2)} - \frac{\|\bx\|^2 I}{B^2(1+\|\bx\|^2/B^2)} + \frac{2\|\bx\|^2 \bx\bx^\top}{B^4(1+\|\bx\|^2/B^2)^2}\right)~.
\end{align*}

\item
\begin{align*}
\frac{\|\bx\|}{\sqrt{1+\|\bx\|^2/B^2}}&\le \|\nabla q_{B,d}(\bx)\|\le \frac{3\|\bx\|}{\sqrt{1+\|\bx\|^2/B^2}}\le 3B~.
\end{align*}

\item
\begin{align*}
\|\nabla^2 q_{B,d}(\bx)\|_{\text{op}}&\le \frac{8}{\sqrt{1+\|\bx\|^2/B^2}}\le 8~.
\end{align*}
\end{enumerate}
\end{Proposition}

\section{Proof of Theorem~\ref{thm:nonsmoothogd_coordinate}}

\label{sec:proof_nonsmoothogd_coordinate}

First, we state and prove a theorem analogous to Theorem~\ref{thm:nonsmooth}.
\begin{Theorem}
\label{thm:nonsmooth_coordinate}
Assume $F:\R^d \to \R$ is well-behaved.
In Algorithm~\ref{alg:template_no_epochs}, set $s_n$ to be a random variable sampled uniformly from $[0,1]$. Set $T,K \in \N$ and $M=KT$.
For $i=1, \dots, d$, set $u^k_i = -D_\infty\frac{\sum_{t=1}^T\frac{\partial F\left( \bw^k_t\right)}{\partial x_i}}{\left|\sum_{t=1}^T \frac{\partial F\left( \bw^k_t\right)}{\partial x_i}\right|}$ for some $D_\infty>0$.
Finally, suppose $\text{Var}(g_{n,i})\le \sigma^2_i$ for $i=1, \dots, d$. Then, we have
\begin{align*}
\E&\left[\frac{1}{K}\sum_{k=1}^K \left\|\frac{1}{T}\sum_{t=1}^T \nabla F(\bw^k_t)\right\|_1\right]
\le \frac{F(\bx_0)-F^\star}{D_\infty M} + \frac{\E[R_T(\bu^1,\dots, \bu^K)]}{D_\infty M}+\frac{D_\infty\sum_{i=1}^d \sigma _i}{\sqrt{T}}~.
\end{align*}
\end{Theorem}
\begin{proof}
In Theorem~\ref{thm:otnc}, set $\bu_{n}$ to be equal to $\bu^1$ for the first $T$ iterations, $\bu^2$ for the second $T$ iterations and so on. In other words, $\bu_n=\bu^{mod(n,T)+1}$ for $n=1, \dots,N$.

From Theorem~\ref{thm:otnc}, we have
\begin{align*}
    \E[F(\bx_M)]&= F(\bx_0) + \E[R_T(\bu^1,\dots,\bu^K)] + \E\left[\sum_{n=1}^M\langle \bg_n, \bu_n\rangle\right]~.
\end{align*}
Now, since $u_{k,i} = -D_\infty\frac{\sum_{t=1}^T\frac{\partial F\left( \bw^k_t\right)}{\partial x_i}}{\left|\sum_{t=1}^T \frac{\partial F\left( \bw^k_t\right)}{\partial x_i}\right|}$, $\E[\bg_n] = \nabla F(\bw_n)$, and $\text{Var}(g_{n,i})\le \sigma^2_i$ for $i=1,\dots, d$, we have
\begin{align*}
    \E\left[\sum_{n=1}^M\langle \bg_n, \bu_n\rangle \right]
    &\le \E\left[\sum_{k=1}^K\left\langle \sum_{t=1}^T \nabla F(\bw_t^k), \bu^k\right\rangle + D_\infty \sum_{k=1}^K\left\|\sum_{t=1}^T (\nabla F(\bw_t^k) -\bg_n)\right\|_1 \right]\\
    &\le \E\left[\sum_{k=1}^K\left\langle \sum_{t=1}^T \nabla F(\bw_t^k), \bu^k\right\rangle\right] + D_\infty K\sqrt{T} \sum_{i=1}^d \sigma _i\\
    &= \E\left[-\sum_{k=1}^K D_\infty T\left\|\frac{1}{T}\sum_{t=1}^T \nabla F\left( \bw^k_t\right)\right\|_1\right] + D_\infty K\sqrt{T} \sum_{i=1}^d \sigma _i~.
\end{align*}
Putting this all together, we have
\[
F^\star
\leq \E[F(\bx_N)]
\le  F(\bx_0) + \E[R_T(\bu^1,\dots,\bu^k)]+D_\infty K\sqrt{T} \sum_{i=1}^d \sigma _i -D_\infty T \sum_{k=1}^K\E\left[\left\|\frac{1}{T}\sum_{t=1}^T \nabla F\left( \bw^k_t\right)\right\|_1\right]~.
\]
Dividing by $K T D_\infty = D_\infty M$ and reordering, we have the stated bound.
\end{proof}

We can now prove Theorem~\ref{thm:nonsmoothogd_coordinate}.
\begin{proof}[Proof of Theorem~\ref{thm:nonsmoothogd_coordinate}]
Since $\A$ guarantees $\|\bDelta_n\|_\infty\le D_\infty$, for all $n<n'\leq T+n-1$, we have
\begin{align*}
\|\bw_n - \bw_{n'}\|_\infty&=\|\bx_{n}-(1-s_n)\bDelta_n -\bx_{n'-1} + s_{n'}\bDelta_{n'}\|_\infty\\
&\le \left\|\sum_{i=n+1}^{n'-1} \bDelta_i\right\|_\infty +\|\bDelta_n\|_\infty+\|\bDelta_{n'}\|_\infty\\
&\le D_\infty((n'-1) -(n+1) +1) + 2D_\infty \\
&= D_\infty(n'-n+1)\\
&\le D_\infty T ~.
\end{align*}
Therefore, we clearly have $\|\bw^k_t-\overline{\bw}^k\|_\infty\le D_\infty T=\delta$.

Note that from the choice of $K$ and $T$ we have $M=KT\geq N-T\geq N/2$. Now, observe that  $\text{Var}(g_{n,i})\le \E[g_{n_i}^2]\le G_i^2$. Thus, applying Theorem~\ref{thm:nonsmooth_coordinate} in concert with the additional assumption $\E[R_T(\bu^1,\dots,\bu^K)]\le D_\infty K \sqrt{T} \sum_{i=1}^d G_i$, we have
\begin{align*}
\E\left[\frac{1}{K}\sum_{k=1}^K \left\|\frac{1}{T}\sum_{t=1}^T \nabla F(\bw^k_t)\right\|_1\right]
&\le 2\frac{F(\bx_0)-F^\star}{D_\infty N} + 2\frac{K D_\infty \sqrt{T} \sum_{i=1}^d G_i}{D_\infty N}+\frac{\sum_{i=1}^d G_i}{\sqrt{T}}\\
&=\frac{2T(F(\bx_0)-F^\star)}{\delta N} +\frac{3\sum_{i=1}^d G_i}{\sqrt{T}}\\
&\le \max\left(\frac{5(\sum_{i=1}^d G_i)^{2/3}(F(\bx_0)-F^\star)^{1/3}}{(N\delta)^{1/3}}, \frac{6 \sum_{i=1}^d G_i}{\sqrt{N}}\right)+ \frac{2(F(\bx_0)-F^\star)}{\delta N},
\end{align*}
where the last inequality is due to the choice of $T$.

Now to conclude, observe that $\|\bw^k_t-\overline{\bw}^k\|_\infty\le \delta$ for all $t$ and $k$, and also that $\overline{\bw}^k =\frac{1}{T}\sum_{t=1}^T \bw^k_t$. Therefore $\bgamma=\{\bw^k_1,\dots,\bw^k_T\}$ satisfies the conditions in the infimum in Definition~\ref{def:critmeasure_coordinate} so that $\|\nabla F(\overline{\bw}^k)\|_{1,\delta} \le \left\|\frac{1}{T}\sum_{t=1}^T \nabla F(\bw^k_t)\right\|_1$.
\end{proof}

\section{Directional Derivative Setting}\label{sec:directional}

In the main text, our algorithms make use of a stochastic gradient oracle. However, the prior work of \cite{zhang2020complexity} instead considers a stochastic  \emph{directional} gradient oracle. This is a less common setup, and other works (e.g., \cite{davis2021gradient}) have also taken our route of tackling non-smooth optimization via an oracle that returns gradients at points of differentiability.

Nevertheless, all our results extend easily to the exact setting of \cite{zhang2020complexity} in which $F$ is  Lipschitz and directionally differentiable and we have access to  a stochastic directional gradient oracle rather than a stochastic gradient oracle. To quantify this setting, we need a bit more notation which we copy directly from \cite{zhang2020complexity} below:

First, from \cite{clarke1990optimization} and \cite{zhang2020complexity}, the \emph{generalized directional derivative} of a  function $F$ in a direction $\bd$ is
\begin{align}
F^{\circ}(\bx, \bd) = \limsup_{\by\to  \bx\ t\downarrow 0}  \frac{f(\by + t\bd) - f(\by)}{t}~. \label{eqn:clarke}
\end{align}
Further, the \emph{generalized gradient} is the set
\begin{align*}
\partial F(\bx) = \{\bg:\ \langle  \bg,\bd\rangle \le \langle F^{\circ}(\bx, \bd),\bd\rangle\text{ for all }\bd\}~.
\end{align*}

Finally, $F:\R^d\to \R$ is \emph{Hadamard directionally differentiable} in the direction $\bv\in \R^d$ if for any function $\psi:\R_+\to \R^d$ such that $\lim_{t\to 0}\frac{\psi(t)-\psi(0)}{t}=\bv$ and $\psi(0)=\bx$, the following limit exists:
\begin{align*}
    \lim_{t\to 0} \frac{F(\psi(t)) -F(\bx)}{t}~.
\end{align*}
If $F$ is Hadamard directionally differentiable, then the  above limit is denoted $F'(\bx, \bv)$. When $F$ is Hadamard directionally differentiable for all $\bx$ and $\bv$, then we  say simply that $F$ is directionally differentiable.

With these definitions, a \emph{stochastic directional oracle} for a Lipschitz, directionally differentiable, and bounded from below function $F$ is an oracle $\O(\bx, \bv, \bz)$ that outputs $\bg\in \partial  F(\bx)$ such that $\langle \bg,  \bv\rangle=F'(\bx,\bv)$. In this case, \cite{zhang2020complexity} shows (Lemma 3) that $F$ satisfies an alternative notion of well-behavedness: 
\begin{align}
    F(\by)-F(\bx)  = \int_0^1 \langle \E[\O(\bx+t(\by-\bx),\by-\bx,\bz)],\by-\bx\rangle dt~.\label{eqn:wellbehaved_direction}
\end{align}

Next, we define:

\begin{Definition}\label{def:crit_direction}
A point $\bx$ is a $(\delta,\epsilon)$ stationary point of $F$ for the generalized gradient if there is a set of points $\bgamma$ contained in the ball of radius $\delta$ centered at $\bx$ such that for $\by$ selected uniformly at random from $\bgamma$, $\E[\by]=\bx$ and for all $\by$ there is a choice of $\bg_\by\in \partial  F(\by)$ such that $\|\E[\bg_\by]\|\le \epsilon$.
\end{Definition}

Similarly, we have the definition:
\begin{Definition}\label{def:critmeasure_direction}
Given  a point $\bx$, and a number $\delta>0$, define:
\begin{align*}
    \|\partial F(\bx)\|_\delta \triangleq \inf_{\bgamma\subset B(\bx, \delta), \frac{1}{|\bgamma|}\sum_{\by\in  \bgamma}\by=\bx, \bg_\by\in \partial F(\by)}\left\|\frac{1}{|\bgamma|}\sum_{\by\in \bgamma} \bg_\by\right\|~.
\end{align*}
\end{Definition}

In fact, whenever a locally Lipschitz function $F$ is differentiable at a point $\bx$, we have that $\nabla F(\bx)\in \partial F(\bx)$, so that $\|\partial F(\bx)\|_\delta \le \|\nabla F(\bx)\|_\delta$. Thus our results  in the main text also  bound $\|\partial F(\bx)\|_\delta$. However, while a gradient oracle is also directional derivative oracle, a directional derivative oracle is only guaranteed to be a gradient oracle if $F$ is \emph{continuously} differentiable at the queried point $\bx$. This technical issue means that when we have access to a directional derivative oracle rather than a gradient oracle, we will instead only bound $\|\partial F(\bx)\|_\delta$ rather than $\|\nabla F(\bx)\|_\delta$.

Despite this technical complication, our overall strategy is essentially identical. The key observation is that the only time at which we used the properties of the gradient previously was when we invoked well-behavedness of $F$. When we have a directional derivative instead of the gradient, the alternative notion of well-behavedness in (\ref{eqn:wellbehaved_direction}) will play an identical role. Thus, our approach is simply to replace the call to $\O(\bw_n, \bz_n)$ in Algorithm~\ref{alg:template_no_epochs} with a call instead to $\O(\bw_n,  \bDelta_n, \bz_n)$ (see Algorithm~\ref{alg:template_no_epochs_direction}). With this change, all of our analysis in the main text applies almost without modification. Essentially, we only need to change notation in a few places to reflect the updated definitions.

\begin{algorithm}[t]
   \caption{Online-to-Non-Convex Conversion (directional derivative oracle version)}
   \label{alg:template_no_epochs_direction}
  \begin{algorithmic}
      \STATE{\bfseries Input: } Initial point $\bx_0$, $K \in\N$, $T \in \N$, online learning algorithm $\A$, $s_n$ for all $n$
      \STATE Set $M=K \cdot T$
      \FOR{$n=1\dots M$}
      \STATE Get $\bDelta_n$ from $\A$
      \STATE Set $\bx_n =\bx_{n-1}+\bDelta_n$
      \STATE Set $\bw_n = \bx_{n-1} + s_n\bDelta_n$
      \STATE Sample random $\bz_n$
      \STATE Generate directional derivative $\bg_n= \O(\bw_n, \bDelta_n, \bz_n)$
      \STATE Send $\bg_n$ to $\A$ as gradient
      \ENDFOR
      \STATE Set $\bw^k_t=\bw_{(k-1)T+t}$ for $k=1,\dots, K$ and $t=1,\dots, T$
      \STATE Set $\overline{\bw}^k = \frac{1}{T}\sum_{t=1}^T \bw^k_t$  for $k=1,\dots, K$
      \STATE{\bfseries Return} $\{\overline{\bw}^1, \dots, \overline{\bw}^K\}$
   \end{algorithmic}
\end{algorithm}

To begin this notational update, the counterpart to Theorem~\ref{thm:otnc} is:

\begin{Theorem}
\label{thm:otnc_directional}
Suppose $F$ is Lipschitz and directionally differentiable. 
With the notation in Algorithm~\ref{alg:template_no_epochs_direction}, if we let $s_n$ be independent random variables uniformly distributed in $[0,1]$, then for any sequence of vectors $\bu_1, \dots, \bu_N$, if  we have the \emph{equality}:
\begin{align*}
    \E[F(\bx_M)]&= F(\bx_0) + \E\left[\sum_{n=1}^M \langle \bg_n, \bDelta_n - \bu_n\rangle\right] 
    + \E\left[\sum_{n=1}^M \langle \bg_n, \bu_n\rangle\right]~.
\end{align*}
\end{Theorem}
\begin{proof}
\begin{align*}
    F
    (\bx_n) - F(\bx_{n-1}) 
    &=\int_0^1 \! \langle \E[\O(\bx_{n-1} + s(\bx_n - \bx_{n-1}),\bx_n-\bx_{n-1},\bz_n)],\bx_n-\bx_{n-1}\rangle \, \mathrm{d}s\\
    &= \E[\langle \bg_n, \bDelta_n\rangle]\\
    &= \E[\langle \bg_n, \bDelta_n-\bu_n\rangle  + \langle \bg_n,\bu_n\rangle]~.
\end{align*}
Where in  the second line we have used the definition  $\bg_n=\O(\bx_{n-1}+s_n(\bx_n-\bx_{n-1}),\bx_n-\bx_{n-1},\bz_n)$, the assumption that $s_n$ is uniform on $[0,1]$, and Fubini theorem (as $\O$ is bounded by Lipschitzness of $F$).
Now, sum over $n$ and telescope to obtain the stated bound.

\end{proof}

Next, we have the following analog of Theorem~\ref{thm:nonsmooth}:

\begin{Theorem}
\label{thm:nonsmooth_direction}
With the notation in  Algorithm~\ref{alg:template_no_epochs_direction}, set $s_n$ to be a random variable sampled uniformly from $[0,1]$. Set $T,K \in \N$ and $M=KT$. Define $\bnabla_t^k = \E[\bg_{(k-1)T+t}]$.
Define $\bu^k = -D\frac{\sum_{t=1}^T\bnabla_t^k}{\left\|\sum_{t=1}^T \bnabla_t^k\right\|}$ for some $D>0$ for $k=1,\dots,K$.
Finally, suppose $\text{Var}(\bg_n)= \sigma^2$. Then:
\begin{align*}
\E
\left[\frac{1}{K}\sum_{k=1}^K \left\|\frac{1}{T}\sum_{t=1}^T \bnabla_t^k\right\|\right]
&\le \frac{F(\bx_0)-F^\star}{D M} + \frac{\E[R_T(\bu^1,\dots,\bu^K)]}{D M}+\frac{\sigma}{\sqrt{T}}~.
\end{align*}
\end{Theorem}

\begin{proof}
The proof is essentially identical to that of Theorem~\ref{thm:nonsmooth}.
In Theorem~\ref{thm:otnc_directional}, set $\bu_{n}$ to be equal to $\bu^1$ for the first $T$ iterations, $\bu^2$ for the second $T$ iterations and so on. In other words, $\bu_n=\bu^{mod(n,T)+1}$ for $n=1, \dots,M$.
So, we have
\begin{align*}
    \E[F(\bx_M)]
    &= F(\bx_0) + \E[R_T(\bu^1,\dots,\bu^K)]
    + \E\left[\sum_{n=1}^M\langle \bg_n, \bu_n\rangle\right]~.
\end{align*}
Now, since $\bu^k=-D \frac{\sum_{t=1}^T\bnabla_t^k}{\left\|\sum_{t=1}^T \bnabla_t^k\right\|}$, and $\text{Var}(\bg_n)= \sigma^2$, we have
\begin{align*}
    \E\left[\sum_{n=1}^M\langle \bg_n, \bu_n\rangle\right]
    &
    \le \E\left[\sum_{k=1}^K\left\langle \sum_{t=1}^T \bnabla_t^k, \bu^k\right\rangle\right]
    + \E\left[D \sum_{k=1}^K\left\|\sum_{t=1}^T(\bnabla_t^k -\bg_{(k-1)T+t})\right\| \right]\\
    &\le \E\left[\sum_{k=1}^K\left\langle \sum_{t=1}^T \bnabla_t^k, \bu^k\right\rangle\right] +D \sigma K\sqrt{T}\\
    &= \E\left[-\sum_{k=1}^K DT\left\|\frac{1}{T}\sum_{t=1}^T \bnabla_t^k\right\|\right] +D \sigma K\sqrt{T}~.
\end{align*}

Putting this all together, we have
\begin{align*}
F^\star 
&\leq \E[F(\bx_M)]
\le  F(\bx_0) + \E[R_T(\bu^1,\dots, \bu^K)] 
+\sigma DK\sqrt{T}-DT \sum_{k=1}^K\E\left[\left\|\frac{1}{T}\sum_{t=1}^T \nabla F\left( \bw^k_t\right)\right\|\right]~.
\end{align*}
Dividing by $K D T = D M$ and reordering, we have the stated bound.
\end{proof}

Finally, we instantiate Theorem~\ref{thm:nonsmooth_direction} with online gradient descent to obtain the analog of Corollary~\ref{thm:nonsmoothogd}. This result establishes that the online-to-batch conversion finds an $(\delta,\epsilon)$ critical point in $O(1/\epsilon^3\delta)$ iterations, even when using a directional derivative oracle. Further, our lower bound construction makes use of  continuously differentiable functions, for which the directional derivative oracle and the standard gradient oracle must coincide. Thus the $O(1/\epsilon^3\delta)$ complexity is optimal in this setting as well.

\begin{Corollary}\label{thm:nonsmoothogd_direction}
Suppose we have a budget of $N$ gradient evaluations.
Under the assumptions and notation of Theorem~\ref{thm:nonsmooth_direction}, suppose in addition $\E[\|\bg_n\|^2]\le G^2$ and that $\A$ guarantees $\|\bDelta_n\|\le D$ for some user-specified $D$ for all $n$ and ensures the worst-case $K$-shifting regret bound $\E[R_T(\bu^1, \dots, \bu^K)]\le DGK\sqrt{T}$ for all $\|\bu^k\|\le D$ (e.g., as achieved by the OGD algorithm that is reset every $T$ iterations). Let $\delta>0$ be an arbitrary number. Set $D=\delta/T$, $T=\min(\lceil(\frac{GN\delta}{F(\bx_0)-F^\star})^{2/3}\rceil,\frac{N}{2})$, and $K=\lfloor\frac{N}{T}\rfloor$.
Then,
for all $k$ and $t$, $\|\overline{\bw}^k-\bw^k_t\|\le \delta$.

Moreover, we have the inequality
\begin{align*}
\E&\left[\frac{1}{K}\sum_{k=1}^K \left\|\frac{1}{T}\sum_{t=1}^T \bnabla_t^k\right\|\right] \leq \frac{2(F(\bx_0)-F^\star)}{\delta N}
+ \max\left(\frac{5G^{2/3}(F(\bx_0)-F^\star)^{1/3}}{(N\delta)^{1/3}},\frac{6G}{\sqrt{N}}\right), 
\end{align*}
which implies
\begin{align*}
\frac{1}{K}&\sum_{t=1}^K \|\partial F(\overline{\bw}^k)\|_\delta\leq \frac{2(F(\bx_0)-F^\star)}{\delta N}
+ \max\left(\frac{5G^{2/3}(F(\bx_0)-F^\star)^{1/3}}{(N\delta)^{1/3}},\frac{6G}{\sqrt{N}}\right)~. 
\end{align*}
\end{Corollary}

\begin{proof}
Since $\A$ guarantees $\|\bDelta_n\|\le D$, for all $n<n'\leq T+n-1$, we have
\begin{align*}
\|\bw_n - \bw_{n'}\|&=\|\bx_{n}-(1-s_n)\bDelta_n -\bx_{n'-1} + s_{n'}\bDelta_{n'}\|\\
&\le \left\|\sum_{i=n+1}^{n'-1} \bDelta_i\right\| +\|\bDelta_n\|+\|\bDelta_{n'}\|\\
&\le D((n'-1) -(n+1) +1) + 2D\\
&= D(n'-n+1)
\le DT ~.
\end{align*}
Therefore, we clearly have $\|\bw^k_t-\overline{\bw}^k\|\le DT=\delta$.

Note that from the choice of $K$ and $T$ we have $M=KT\geq N-T\geq N/2$.
So, for the second fact, notice that $\text{Var}(\bg_n)\le G^2$ for all $n$. Thus, applying Theorem~\ref{thm:nonsmooth_direction} in concert with the additional assumption $\E[R_T(\bu^1,\dots, \bu^K)]\le DGK\sqrt{T}$, we have:
\begin{align*}
\E
\left[\frac{1}{K}\sum_{k=1}^K \left\|\frac{1}{T}\sum_{t=1}^T \bnabla_t^k\right\|\right]
&\le 2\frac{F(\bx_0)-F^\star}{DN} + 2\frac{KDG\sqrt{T}}{DN}+\frac{G}{\sqrt{T}}\\
&\le \frac{2T(F(\bx_0)-F^\star)}{\delta N} +\frac{3G}{\sqrt{T}}\\
&\le \max\left(\frac{5G^{2/3}(F(\bx_0)-F^\star)^{1/3}}{(N\delta)^{1/3}}, \frac{6G}{\sqrt{N}}\right)
+ \frac{2(F(\bx_0)-F^\star)}{\delta N},
\end{align*}
where the last inequality is due to the choice of $T$.

Finally, observe that $\|\bw^k_t-\overline{\bw}^k\|\le \delta$ for all $t$ and $k$, and also that $\overline{\bw}^k =\frac{1}{T}\sum_{t=1}^T \bw^k_t$. Therefore $\bgamma=\{\bw^k_1,\dots,\bw^k_T\}$ satisfies the conditions in the infimum in Definition~\ref{def:critmeasure_direction} so that $\|\partial F(\overline{\bw}^k)\|_\delta \le \left\|\frac{1}{T}\sum_{t=1}^T \bnabla_t^k\right\|$.
\end{proof}

\end{document}